\documentclass[12pt]{article}

\usepackage{PRIMEarxiv}
\usepackage[a4paper,
            bindingoffset=0.2in,
            left=1in,
            right=1in,
            top=1in,
            bottom=1in,
            footskip=.25in]{geometry}
\usepackage{textcomp}
\usepackage{makecell}
\usepackage{longtable}
\usepackage[utf8]{inputenc} 
\usepackage[T1]{fontenc}    
\usepackage[colorlinks=true,citecolor=blue]{hyperref}       
\usepackage{url}            
\usepackage{booktabs}       
\usepackage{amsfonts}       
\usepackage{nicefrac}       
\usepackage{microtype}      
\usepackage{lipsum}
\usepackage{fancyhdr}       
\usepackage{graphicx}       
\usepackage[pagewise]{lineno}
\graphicspath{{media/}}     

\usepackage{amsmath}
\numberwithin{equation}{section}
\usepackage{natbib}

\usepackage{enumitem}
\usepackage{setspace} 
\usepackage{multirow}
\usepackage[table,xcdraw]{xcolor}
\usepackage{float}
\usepackage[draft]{pgf}

\usepackage{stmaryrd}

\usepackage{multirow}

\usepackage[section]{placeins}
\usepackage{array}
\usepackage{booktabs}
\usepackage{subcaption}
\usepackage{etoolbox}
\AfterEndEnvironment{enumerate}{\vskip-\lastskip}
\usepackage[noend]{algpseudocode}
\usepackage[ruled]{algorithm}
\usepackage{chngcntr}
\usepackage[noabbrev,nameinlink]{cleveref}
\usepackage{rotating}
\usepackage{tikz}
\usepackage{amsthm}
\theoremstyle{definition}
\newtheorem{definition}{Definition}
\newtheorem{theorem}{Theorem}

\usepackage{mathtools}
\usepackage{tablefootnote}
\newcommand{\tabitem}{~~\llap{\textbullet}~~}

\makeatletter
\newenvironment{breakablealgorithm}
  {
   \begin{center}
     \refstepcounter{algorithm}
     \hrule height.8pt depth0pt \kern2pt
     \renewcommand{\caption}[2][\relax]{
       {\raggedright\textbf{\fname@algorithm~\thealgorithm} ##2\par}%
       \ifx\relax##1\relax 
         \addcontentsline{loa}{algorithm}{\protect\numberline{\thealgorithm}##2}%
       \else 
         \addcontentsline{loa}{algorithm}{\protect\numberline{\thealgorithm}##1}%
       \fi
       \kern2pt\hrule\kern2pt
     }
  }{
     \kern2pt\hrule\relax
   \end{center}
  }
\makeatother

\usepackage{titlesec}
\titlespacing*{\section}{0pt}{1ex}{1ex}
\titlespacing*{\subsection}{0pt}{1ex}{1ex}
\titlespacing*{\subsubsection}{0pt}{1ex}{1ex}
\usepackage[T1]{fontenc}
\usepackage{newtxmath,newtxtext}
\usepackage{array}  
\usepackage{graphicx}  

\newcolumntype{L}[1]{>{\raggedright\arraybackslash}p{#1}}
\newcolumntype{C}[1]{>{\centering\arraybackslash}p{#1}}
\newcolumntype{R}[1]{>{\raggedleft\arraybackslash}p{#1}}

\allowdisplaybreaks

\title{\bf{Time-varying Factor Augmented Vector Autoregression with Grouped Sparse Autoencoder}}

\author{
Yiyong Luo$^1$ \\
yiyong.luo.20@ucl.ac.uk
\And 
Brooks Paige$^2$ \\
b.paige@ucl.ac.uk
\And
Jim E. Griffin$^1$ \\
j.griffin@ucl.ac.uk
  \vspace{1cm}
 \AND{
 $^1$ \normalsize{Department of Statistical Science, University College London}}
 \AND{
 $^2$ \normalsize{AI Centre, Department of Computer Science, University College London}
 }
}

\begin{document}
\maketitle

\begin{abstract}
Recent economic events, including the global financial crisis and COVID-19 pandemic, have exposed limitations in linear Factor Augmented Vector Autoregressive (FAVAR) models for forecasting and structural analysis. Nonlinear dimension techniques, particularly autoencoders, have emerged as promising alternatives in a FAVAR framework, but challenges remain in identifiability, interpretability, and integration with traditional nonlinear time series methods. We address these challenges through two contributions. First, we introduce a Grouped Sparse autoencoder that employs the Spike-and-Slab Lasso prior, with parameters under this prior being shared across variables of the same economic category, thereby achieving semi-identifiability and enhancing model interpretability. Second, we incorporate time-varying parameters into the VAR component to better capture evolving economic dynamics. Our empirical application to the US economy demonstrates that the Grouped Sparse autoencoder produces more interpretable factors through its parsimonious structure; and its combination with time-varying parameter VAR shows superior performance in both point and density forecasting. Impulse response analysis reveals that monetary policy shocks during recessions generate more moderate responses with higher uncertainty compared to expansionary periods.
\end{abstract}


\setlength{\belowdisplayskip}{2pt} \setlength{\belowdisplayshortskip}{2pt}
\setlength{\abovedisplayskip}{2pt} \setlength{\abovedisplayshortskip}{2pt}

\section{Introduction}
\label{sec:Introduction}

Factor-augmented vector autoregression (FAVAR, \cite{bernanke2005measuring}) is a widely-used model to study high-dimensional macroeconomic time series. The FAVAR enhances a standard Vector Autoregression (VAR) by incorporating both observable factors (key economic indicators) and latent factors. Using linear dimension reduction techniques such as Principal Component Analysis (PCA), these latent factors - typically numbering fewer than 10 - are extracted from hundreds of time series, enabling the FAVAR to include a large amount of information without suffering the curse of dimensionality that would occur if all these time series were directly modeled in the VAR. Moreover, the FAVAR allows analysis of both observable factors in the VAR and the high-dimensional time series due to the factor structure. 

Numerous studies (\cite{blanchard2001long}, \cite{mumtaz2011lies}, \cite{sims2006were}, \cite{stock2002has}, among others) have demonstrated that the transmission mechanisms of monetary policy and volatility of economic shocks change over time. Additionally, events such as the Global Financial Crisis (GFC) and the COVID-19 pandemic can lead to sudden shifts in variable levels and introduce outliers. These temporal changes and structural breaks pose challenges for the FAVAR, as both the dimension reduction and the VAR are linear and thus inherently limited in capturing such complex, time-varying dynamics. To address these limitations, various approaches draw from the time series literature to incorporate non-linear structures into the FAVAR framework. For the dimension reduction part, \cite{hacioglu2016interpreting} extended the threshold factor model \citep{nakajima2013bayesian} to propose the threshold FAVAR. For the VAR part,  \cite{korobilis2013assessing} introduced time-varying parameters (TVP) to the FAVAR, while \cite{huber2018markov} implemented a Markov switching VAR inspired by \cite{hamilton1989new}. Taking a holistic perspective, \cite{koop2014new} and \cite{abbate2016changing} modeled the evolution of parameters in both parts as random walks. Apart from these time series-based approaches, \cite{klieber2024non} leveraged deep learning techniques, specifically autoencoders, to extract factors. This non-linear FAVAR showed robust performance in handling outliers during the COVID-19 pandemic.

Given the growing popularity of autoencoders in econometric research \citep{cabanilla2019forecasting, andreini2020deep, hauzenberger2023real}, we focus on the FAVAR with an autoencoder and highlight three aspects that require attention and further improvement. Firstly, interpreting the latent factors and applying them to downstream tasks, such as the impulse response analysis in the FAVAR, necessitate identifiable factors; however,
factors extracted from a standard autoencoder generally do not satisfy this requirement \citep{locatello2019challenging}. Secondly, even when the factors are identified, determining their economic meanings remains challenging due to two issues: 1) the black box nature of the autoencoder, and 2) the lack of parsimony. While post-processing interpretation frameworks such as Shapley additive explanations \citep{strumbelj2010efficient} can be applied to latent factors extracted from the autoencoder, the complexity of the autoencoder often leads to results suggesting that each factor has a non-negligible impact on most of the 
high-dimensional time series. This makes it difficult to discern the specific economic role of individual factors. Lastly, \cite{klieber2024non} assumed a time-invariant VAR, resulting in constant monetary policy effects and homoscedasticity. One can relax this assumption to accommodate time-varying dynamics of the economy.

In this paper, our first contribution tackles the identifiability and interpretability issues in the first two aspects through sparsity. Inspired by \cite{moran2021identifiable}, we propose a variant of the standard autoencoder, namely the Grouped Sparse (GS) autoencoder. While a standard autoencoder extracts latent factors using the encoder and reconstructs high-dimensional data with the decoder, our approach introduces an intermediate step. We first group the high-dimensional data, then element-wisely multiply these factors by a set of group-specific parameters before passing factors to the decoder. Following the Spike-and-Slab Lasso (SSL, \cite{rovckova2018spike}) prior, these parameters are the same across variables within each group during reconstruction. We prove that the factors are identifiable up to an element-wise transformation, given known \textit{anchor groups} - groups of time series reconstructed by only one factor. By exploiting the properties of these element-wise transformations, we determine the decoder architecture and its activation function. For interpretation, the SSL parameters can effectively activate or deactivate each factor when reconstructing a specific group of data, providing clear economic meanings to the factors and eliminating the need for post-processing interpretation approaches. Our second contribution builds upon the third aspect mentioned earlier. Specifically, we extend the non-linearity to the VAR part of the FAVAR by adopting the time-varying parameter VAR (TVP-VAR, \cite{primiceri2005time}), allowing the VAR parameters to evolve as random walks. 

In our empirical application to the US economy, we first compare factors extracted by the GS autoencoder with those obtained through PCA. The GS autoencoder factors exhibit superior interpretability due to their parsimonious structure. Examining correlations between factors and high-dimensional time series reveals that each factor from the GS autoencoder shows a stronger correlation with data in its corresponding anchor group compared to non-anchor groups. Assessment of point and density forecasting performance demonstrates that the GS autoencoder combined with the TVP-VAR outperforms models using either linear dimension reduction methods or time-invariant VAR parameters. Our impulse response analysis shows that monetary policy shocks during recessions yield more moderate responses and higher uncertainty than those in expansions. Furthermore, these impulse responses of the high-dimensional data present notable time variations.

The paper is organized as follows. Section 2 provides the background on the FAVAR, autoencoder and its challenges. Section 3 introduces the Grouped Sparse (GS) autoencoder and details of the TVP-VAR. Section 4 gives the details about parameter estimation. Section 5 starts with a description of the data and implementation details, then demonstrates the merits of our proposed model in three areas: factor interpretation, forecasting performance, and impulse response analysis. Section 6 concludes the paper.



\section{Background} 
\label{sec:Background}

\subsection{FAVAR} 
\label{sec:Linear FAVAR}

Let $\boldsymbol{x}_t\in\mathbb{R}^N$ be the observable high-dimensional data at time point $t$, the FAVAR represents $\boldsymbol{x}_t$ as low-dimensional factors: $\left(\boldsymbol{f}^\prime, \boldsymbol{y}^\prime_t\right)^\prime$, where $\boldsymbol{f}_t\in\mathbb{R}^K$ contains latent factors and $\boldsymbol{y}_t\in\mathbb{R}^M$ represents key economic indicators as observable factors. The evolution of these latent and observable factors follows a VAR. Mathematically, the FAVAR is as follows:
\begin{align}
\left(
\begin{aligned}
&\boldsymbol{x}_t\\
&\boldsymbol{y}_t 
\end{aligned} 
\right) &= 
    \begin{pmatrix}
    \multicolumn{2}{c}{\boldsymbol{\Lambda}}\\
    \mathbf{0} & \boldsymbol{I}
    \end{pmatrix}
 \left(
\begin{aligned}
&\boldsymbol{f}_t\\
&\boldsymbol{y}_t 
\end{aligned} 
\right) + \left(
\begin{aligned}
&\boldsymbol{\epsilon}_t \\
& \mathbf{0} 
\end{aligned} 
\right),\, \boldsymbol{\epsilon}_t\sim\mathcal{N}(\textbf{0},\boldsymbol\Sigma), \label{equ:factor model}\\
\left(
\begin{aligned}
&\boldsymbol{f}_t\\
&\boldsymbol{y}_t 
\end{aligned} 
\right) &= \boldsymbol{A}_{1} \left(
\begin{aligned}
&\boldsymbol{f}_{t-1}\\
&\boldsymbol{y}_{t-1} 
\end{aligned} 
\right) + \dots+\boldsymbol{A}_{P} \left(
\begin{aligned}
&\boldsymbol{f}_{t-P}\\
&\boldsymbol{y}_{t-P} 
\end{aligned} 
\right) +\boldsymbol{\eta}_t,\, \boldsymbol{\eta}_t\sim\mathcal{N}(\textbf{0},\boldsymbol\Omega),
\label{equ:FAVAR}
\end{align}
where $t=1,\dots,T$, $\boldsymbol{\Lambda}\in\mathbb{R}^{N\times (M+K)}$ is the factor loading, $\boldsymbol{\Sigma}$ and $\boldsymbol{\Omega}$ are the variance-covariance matrices, and $\boldsymbol{A}=\left(\boldsymbol{A}_1,\dots, \boldsymbol{A}_P\right)$ is the coefficient matrix of the VAR with lag order $P$. 

The FAVAR employs a low-dimensional representation of $\boldsymbol{x}_t$ to mitigate the curse of dimensionality in two perspectives. First, it allows the FAVAR to incorporate the rich information in $\boldsymbol{x}_t$ when analyzing $\boldsymbol{y}_t$. Second, the established relationship between factors and $\boldsymbol{x}_t$ enables straightforward derivation of impulse responses of $\boldsymbol{x}_t$ based on those of $\left(\boldsymbol{f}^\prime, \boldsymbol{y}^\prime_t\right)^\prime$, avoiding the need to directly model $\boldsymbol{x}_t$ using its past values.

While econometricians typically select the variables in $\boldsymbol{x}_t$ and $\boldsymbol{y}_t$ according to their research goals,
obtaining $\boldsymbol{f}_t$ is more complex.  \cite{bernanke2005measuring} proposed two methods to get these latent factors: one- and two-step procedures. The former adopts the Bayesian inference of dynamic factor model \citep{stock2005implications, stock2011}, which treats the FAVAR as a state space model. The advantage of this procedure is that it incorporates the VAR prior to latent factors during sampling. However, this procedure has two disadvantages. First, the corresponding MCMC is complicated because the sampled Markov chains may exhibit autocorrelation due to the interdependence between factors and loadings in their full conditionals. While longer chains could address this issue, such a solution increases computational costs given the high dimensionality of $\boldsymbol{x}_t$. Second, factors and loadings in the one-step procedure are not identified due to the orthogonal rotation. Specifically, $\boldsymbol{f}^*_t = \boldsymbol{Q}\boldsymbol{f}_t$ and $\boldsymbol{\Lambda}^* = \boldsymbol{\Lambda}\boldsymbol{Q}^\prime$, with $\boldsymbol{Q}^\prime\boldsymbol{Q} = \boldsymbol{I}$, gives the same likelihood as $\boldsymbol{f}_t$ and $\boldsymbol{\Lambda}$. Imposing restrictions to $\boldsymbol{\Lambda}$, such as the upper $K$\texttimes$K$ submatrix is an identity matrix, mitigates this issue, but one needs to determine which $K$ variables in $\boldsymbol{x}_t$ are driven by only one factor. In contrast, the two-step procedure circumvents both autocorrelation and identifiability issues by extracting factors as principal components and then using these factors to infer VAR parameters in \labelcref{equ:FAVAR}. This procedure avoids the autocorrelation issue, as latent factors are not part of the MCMC sampling. Identifiability is readily achieved by imposing $\boldsymbol{F}^\prime\boldsymbol{F}=\boldsymbol{I}$, where $\boldsymbol{F} = (\boldsymbol{f}_1,\dots,\boldsymbol{f}_T)^\prime$. The main limitation of this approach is that the latent factors determined in the first step are deterministic, precluding the incorporation of VAR prior information. Nevertheless, \cite{bernanke2005measuring} showed that both procedures yield similar results.

Additional specifications of the FAVAR depend on the strategy to identify economic shocks. In particular, we can impose restrictions to a matrix $\boldsymbol{H}$, which appears in the decomposition of $\boldsymbol{\Omega}$ as $\boldsymbol{H}\boldsymbol{S}\boldsymbol{H}^\prime$. While various identification strategies are available for the FAVAR, see \cite{stock2016dynamic} for a comprehensive review, we focus on the recursive restriction specified in \cite{bernanke2005measuring} due to its simplicity and prevalence. The rationale behind the recursive restriction is that some variables respond contemporaneously to the shock and others respond after one lag, so this restriction refers to the Cholesky decomposition of $\boldsymbol{\Omega}$ to specify $\boldsymbol{H}$ as a lower-triangular matrix with ones on the diagonal, then each factor in $\left(\boldsymbol{f}^\prime, \boldsymbol{y}^\prime_t\right)^\prime$ responds with one lag if it is ordered before the factor corresponding to the shock (\cite{bernanke2005measuring} described this kind of factors fast-moving, and otherwise slow-moving). Due to the order of $\left(\boldsymbol{f}^\prime_t,\boldsymbol{y}^\prime_t\right)^\prime$, one needs to constrain $\boldsymbol{f}_t$ as slow-moving latent factors when the two-step procedure is applied, and the economic shocks of interest are in $\boldsymbol{y}_t$\footnote{Note that this constraint requires no additional specification in the one-step procedure since the inference of $\boldsymbol{f}_t$ already incorporates the restriction in $\boldsymbol{H}$.}. Following \cite{bernanke2005measuring}, this constraint is implemented through a linear regression, $\hat{\boldsymbol{f}_t} = \boldsymbol{B}_{s}\hat{\boldsymbol{f}^s_t} + \boldsymbol{B}_y\boldsymbol{y}_t+\boldsymbol{e}_t$, where $\hat{\boldsymbol{f}_t}$ and $\hat{\boldsymbol{f}^s_t}$ are the 
principal components of $\boldsymbol{x}_t$ and its slow-moving part, respectively, then $\boldsymbol{f}_t$ is constructed as $\hat{\boldsymbol{f}_t}-\boldsymbol{B}_y\boldsymbol{y}_t$. 


\subsection{Autoencoder and its Econometric Applications}
The autoencoder is a deep learning model that compresses the high-dimensional data to a lower dimension. A standard autoencoder comprises three parts: encoder, factors (often referred to as  "bottleneck" in the deep learning terminology), and decoder, with the goal of getting factors from the encoder and reconstructing the high-dimensional data with the factors and decoder. The mathematical expression of an autoencoder is:
\begin{align}
    \boldsymbol{f}_t & = g^e_{\boldsymbol{\phi}} \left(\boldsymbol{x}_t\right)= \left(g^e_L\circ \dots \circ g^e_1\right)_{\boldsymbol{\phi}}
    \left(\boldsymbol{x}_t\right), \label{equ:encoder} \\
    \hat{\boldsymbol{x}}_t & = g^d_{\boldsymbol{\theta}}\left(\boldsymbol{f}_t\right) = \left(g^d_L\circ \dots \circ g^d_1\right)_{\boldsymbol{\theta}}
    \left(\boldsymbol{f}_t\right). \label{equ:decoder}
\end{align}
where $g^e_{\boldsymbol{\phi}}(\cdot)$ and $g^d_{\boldsymbol{\theta}}(\cdot)$ are the encoder and decoder with parameters $\boldsymbol{\phi}$ and $\boldsymbol{\theta}$, respectively, each representing a composition of $L$ functions shown on the right-hand sides, following the convention of the autoencoder literature. The output from each function component is called a "hidden layer" for the first $L-1$ functions and an "output layer" for the final function if we count functions from the right. $\hat{\boldsymbol{x}}_t$ is the reconstruction of $\boldsymbol{x}_t$. 

The first part of the FAVAR with equation \labelcref{equ:factor model} can be seen as a specification of an autoencoder because this equation is simply a linear decoder, replacing $\hat{\boldsymbol{x}}_t$ by $\boldsymbol{x}_t$. This specification does not require the encoder since we can obtain factors using the aforementioned one- or two-step procedures. However, the autoencoder generally requires both encoder and decoder for two reasons. First, if we use the one-step procedure, the full 
conditional of $\boldsymbol{f}_t$ is intractable due to the non-linearity in the decoder, so we need the encoder to facilitate the variational inference of $\boldsymbol{f}_t$\footnote{This extends the autoencoder to Variational autoencoder (VAE, \cite{kingma2014auto})}. Second, using the two-step procedure requires known transformation from $\boldsymbol{x}_t$ to $\boldsymbol{f}_t$, PCA, for example, but the parameters of the decoder are unknown. Moreover, training the decoder necessitates the knowledge of $\boldsymbol{f}_t$, which is also unknown. The encoder-decoder architecture in the autoencoder resolves this issue through unsupervised learning using only $\boldsymbol{x}_t$, enabling simultaneous estimation of both the autoencoder parameters and $\boldsymbol{f}_t$. 

The expressiveness of an autoencoder comes from the essence of all deep learning models: layered architecture and the choice of $g^j_l(\cdot)$, for $j=e$ or $d$ and $l=1,\dots, {L-1}$. Given these function components being non-linear, neural networks can approximate any continuous functions with arbitrary precision, as proved by \cite{LESHNO1993861} and \cite{pinkus1999approximation}, among others. To expand the function compositions in \labelcref{equ:encoder} and \labelcref{equ:decoder}, encoder and decoder can be written in recursive forms with hidden layers. Specifically,
the $l$-th hidden layer at time $t$, $\boldsymbol{h}^{j}_{t,l}$, is a non-linear transformation of the $(l-1)$-th one:
\begin{equation}
    \boldsymbol{h}^j_{t,l}=g^{j}_l\left(\boldsymbol{h}^j_{t,l-1}\right) = g\left(\boldsymbol{W}^{j}_l\boldsymbol{h}^{j}_{t,l-1}+\boldsymbol{b}^{j}_l\right), \label{equ:nonlinearity}
\end{equation}
where $\boldsymbol{h}^e_{t,0}$ and $\boldsymbol{h}^d_{t,0}$ are $\boldsymbol{x}_t$ and $\boldsymbol{f}_t$, respectively, $\boldsymbol{W}^{j}_l\in\mathbb{R}^{D^j_{l}\times D^j_{l-1}}$, $\boldsymbol{b}^{j}_l\in\mathbb{R}^{D^j_l}$ are called weight and bias, which are parameters in $\boldsymbol{\phi}$ or $\boldsymbol{\theta}$, $D^j_{l}$ is the dimension of $\boldsymbol{h}^j_{t,l}$, $g(\cdot)$ is an activation function depending on the specific applications. Hyperbolic tangent, $g\left(x\right)$ = tanh$\left(x\right)$, and rectified linear unit (ReLU), $g\left(x\right)$ = $
\text{max}(0,x)$, are two popular choices. Note that we specify identity activation functions for $g^e_L(\cdot)$ and $g^d_L(\cdot)$, corresponding to the output layers of the encoder and decoder, to be the identity function, making them linear functions, as $\boldsymbol{f}_t$ and $\hat{\boldsymbol{x}}_t$ are real-valued vectors that lie beyond the range of common activation functions.

To learn the autoencoder parameters $\boldsymbol{\phi}$ and $\boldsymbol{\theta}$, one needs to maximize an objective function usually derived from the marginal log-likelihood $p\left(\boldsymbol{x}_{1:T}\right)$. Assume $\boldsymbol{x}_{t}$, for $t=1,\dots, T$, follows a Gaussian distribution with mean $\hat{\boldsymbol{x}}_{t}$ and a diagonal variance-covariance matrix with equal diagonal elements, this objective function is simply the negative 
mean squared error between $\boldsymbol{x}_{1:T}$ and $\hat{\boldsymbol{x}}_{1:T}$. The optimization then updates parameters using gradient descent.

The autoencoder is increasingly popular in econometric studies. In the FAVAR literature, \cite{klieber2024non} adopted the two-step procedure since it is simpler and less time-consuming by replacing the PCA with the autoencoder to extract latent factors. After training the autoencoder, a time-invariant VAR with the Minnesota-type prior \citep{litterman1979techniques} then models the evolution of these latent factors alongside the observable ones. To facilitate further analysis of the impulse responses of $\boldsymbol{x}_t$, this paper used 
the Moore-Penrose pseudoinverse of $\boldsymbol{F}$ to approximate a linear factor loading. Apart from the FAVAR, the autoencoder is applied in econometrics for forecasting. For example, \cite{hauzenberger2023real} extracted factors using the autoencoder, then treated them as covariates to forecast inflation using a linear regression framework. \cite{cabanilla2019forecasting} adopted a similar approach, but used another deep learning architecture to link the factors to the GDP forecasts. 

\subsection{Challenges of Standard Autoencoder and Solutions}
\label{sec:Challenges of Standard Autoencoder and Solutions}

While these applications provide strong examples of combining the standard autoencoder with traditional time series models, one important prerequisite is that the latent factors extracted should be identified, i.e. if two sets of factors $\boldsymbol{f}_t$ and $\boldsymbol{f}^*_t$ give the same $\hat{\boldsymbol{x}}_t$, then $\boldsymbol{f}_t=\boldsymbol{f}^*_t$, for $t=1,\dots, T$. The first challenge of the standard autoencoder is that these factors are generally not identified \citep{locatello2019challenging}. We illustrate this point with two examples:

\textbf{Example 1}. Similar to the rotational invariance in the linear FAVAR, assume $\boldsymbol{f}_t$ and $\boldsymbol{f}^*_t=\boldsymbol{Q}\boldsymbol{f}_t$, for some invertible matrix $\boldsymbol{Q}$, these two sets of factors construct the same $\hat{\boldsymbol{x}}_t$ if the weights in the two autoencoders satisfy: $\boldsymbol{W}^{*,d}_1=\boldsymbol{Q}^{-1}\boldsymbol{W}^{d}_1$ and $\boldsymbol{W}^{*,d}_l=\boldsymbol{W}^{d}_l$, for $l=2,\dots,L$. 

\textbf{Example 2.} If $\boldsymbol{\theta}\neq\boldsymbol{\theta}^*$ and $\boldsymbol{\phi}=\boldsymbol{\phi}^*$, the factors are potentially non-identifiable when the decoders are non-injective. This kind of decoders arises from two cases: 1) when using non-injective activation functions such as ReLU, 2) when the weights do not have full column rank.

Many standard deep learning implementations can result in non-identifiable latent factors. These implementations introduce randomness to the model itself and/or during the training process to explore the parameter space, improve training efficiency and enhance model expressiveness. Examples include, but are not limited to, initializing parameters randomly, stochastic gradient descent \citep{robbins1951stochastic}, and treating latent factors as random variables. Eliminating all these implementations is not ideal because they play a crucial role in deep learning. Thus, many recent efforts in the deep learning literature alleviate the indeterminacy through two streams: 1) modifying the standard decoder structure 
\citep{ moran2021identifiable, lachapelle2024additive}, and 2) imposing well-designed priors on $\boldsymbol{f}_t$ \citep{khemakhem2020variational, lachapelle2022disentanglement}. Both streams aim to 
identify factors up to trivial transformations, such as element-wise transformations and permutations. 

The second challenge of the standard autoencoder lies in its interpretability limitations. An example is in Figure 2 of \cite{klieber2024non}, which compared the importance of latent factors extracted from different dimension reduction methods to the high-dimensional data. The figure presents importance matrices where rows represent data categories and columns denote factors. For the non-linear dimension reduction methods such as the autoencoder, each entry measures factor importance using the Shapley additive explanations framework \citep{strumbelj2010efficient} while factor loadings were used for PCA. The resulting dense importance matrices reveal that latent factors from both methods appear important to almost all categories, demonstrating similar interpretability limitations between PCA and the autoencoder. Two prominent and closely related approaches to enhance the interpretability of the autoencoder are: inducing sparsity in the autoencoder components (typically the latent factors and/or the decoder) \citep{ainsworth2018oi,tank2021neural}, and developing variants of the standard autoencoder that promote disentanglement, where each latent factor represents a distinct meaningful aspect of the high-dimensional data, see \cite{wang2022disentangled} for a review.

\section{Methodology} 
\label{sec:Methodology}

This section describes the non-linearity applied in the two parts of the FAVAR: the factor extraction and VAR parts. Section \ref{sec: Grouped Sparse Autoencoder} introduces a variant of the standard autoencoder, namely the Grouped Sparse (GS) autoencoder, which can alleviate both the identifiability and interpretability challenges aforementioned. Section \ref{sec:FA-TVP-VAR} specifies the TVP-VAR. 

Given the non-linearity in both parts, using the two-step procedure would be much simpler than the one-step procedure to obtain latent factors and learn parameters. This is because the latter treats the latent factors with a VAR prior, so the inferential scheme requires more advanced deep learning frameworks such as Dynamical VAE \citep{giannone2015prior} and Bayesian neural networks \citep{goan2020bayesian}, see the references therein, which extends beyond the scope of this study. 

\subsection{Grouped Sparse Autoencoder}
\label{sec: Grouped Sparse Autoencoder}

We aim to construct a variant of the standard autoencoder that enhances both identifiability and interpretability. The Sparse autoencoder proposed by \cite{moran2021identifiable} provides a promising foundation, as it identifies latent factors up to element-wise transformation and induces sparsity through the SSL prior to decoder parameters (see its description after the mathematical expressions below). To further improve the interpretability of the Sparse autoencoder, we extend it to the Grouped Sparse (GS) autoencoder using group-specific SSL parameters. The grouping effect is justifiable because economic data inherently falls into different categories, with well-established divisions such as labor market, output, and interest rates, among others. In the FAVAR literature, \cite{belviso2006structural} and \cite{korobilis2013assessing} divided data into different groups and extracted each factor from one group.

The GS autoencoder has the same encoder as in \labelcref{equ:encoder}, the mathematical expression of the decoder in \labelcref{equ:decoder} changes to:
\begin{align}
    \hat{x}_{t,i}&= g^d_{i,\boldsymbol{\theta}}\left(\boldsymbol{f}_t\odot {{\boldsymbol{\beta}_{c_i}}}\right)=
    \left(g^d_{i,L}\circ g^d_{L-1}\circ \dots \circ g^d_{1}\right)_{\boldsymbol{\theta}}\left(\boldsymbol{f}_t\odot {{\boldsymbol{\beta}_{c_i}}}\right), \label{equ:sparse_decoder} 
\end{align}
where $i\in \{1,\dots, N\}$ is the index of variables in $\boldsymbol{x}_t$, $g^d_{i,\boldsymbol{\theta}}(\cdot)=\left(g^d_{i,L}\circ g^d_{L-1}\circ \dots \circ g^d_{1}\right)_{\boldsymbol{\theta}}(\cdot)$ is the decoder that reconstructs $\boldsymbol{x}_{t,i}$ 
(this decoder takes a $K$-dimensional input and outputs a scalar), $g^d_{l}$, for $l=1,\dots, L-1$, denotes the $l$-th function in this decoder and is identical across $i$, while $g^d_{i,L}$ is different according to $i$ \footnote{We specify the decoder in this way because it allows us to distinguish the reconstructions of two arbitrary time series from the same data group}, $\boldsymbol{\beta}_{c_i}=(\beta_{c_i,1},\dots,\beta_{c_i,K})\in\mathbb{R}^K$ stores sparsity parameters corresponding to the group of the $i$-th variable, $c_i$ , $\odot$ means element-wise multiplication. 

The distribution that $\boldsymbol{\beta}_{c_i}$ follows is an SSL, so this autoencoder can turn on or off each factor when reconstructing a group of variables. Specifically, 
\begin{align}
    \beta_{c_i, k} &\sim \gamma_{c_i,k}\psi_1(\beta_{c_i, k})+\left(1-\gamma_{c_i,k}\right)\psi_0(\beta_{c_i, k}), \label{equ:SSL} \\
    \gamma_{c_i,k} & \sim \text{Bernoulli}\left(0.5\right) \label{equ:bernoulli}.
\end{align}
where $\psi_s(\beta)=\frac{\lambda_s}{2}\exp(-\lambda_s |\beta|)$ for $s$ = 0 or 1, is the Laplace distribution with $\lambda_0\gg\lambda_1$; $\gamma_{c_i,k}$ is a binary variable that determines the shrinkage level of $\beta_{c_i,k}$, for $k = 1,\dots, K$. If $\gamma_{c_i,k}=0$, the $\beta_{c_i,k}$ is shrunk to zero with a higher probability, and vice versa. 

If we allow a sparsity parameter for each variable (i.e. the group size is one) in \labelcref{equ:sparse_decoder}-\labelcref{equ:bernoulli}, then the model is the Sparse autoencoder. Apart from the sparse autoencoder proposed in \cite{moran2021identifiable}, the autoencoder in \cite{ainsworth2018oi} is also closely related to ours. This autoencoder imposes the Bayesian lasso prior \citep{park2008bayesian} to the weight that produces the first hidden layer in the decoder. The corresponding weight shares the same degree of sparsity for the variables within the same group. The difference between this model and ours is twofold. Firstly, we induce sparsity to ${\boldsymbol{\beta}_{c_i}}$, for $i = 1,\dots,N$, which is different to the weight. We adopt this structure because it facilitates the proof of identifiability that we will discuss later. Secondly, we used the SSL instead of the Bayesian lasso because Lasso \citep{tibshirani1996regression} can under-regularize large coefficients and over-regularize small coefficients \citep{ghosh2016asymptotic}; it also has a sub-optimal posterior
contraction rate in the Bayesian framework \citep{castillo2015bayesian}. The SSL mitigates these issues with theoretical results available in \cite{bai2021spike}. 

Next, we show the latent factors in the GS autoencoder are identifiable up to element-wise transformation. Similar to the model in \cite{moran2021identifiable}, this identifiability assumes that we know the \textit{anchor group} of each factor with the following definition:
\begin{definition}
    A data category $c$ is an anchor group of factor $k$, $k=1,\dots, K$, if $\beta_{c,k}\neq 0$ and $\beta_{c,k^\prime}=0$ for all $k^\prime\neq k$.
\end{definition}
The concept of "anchor" first appeared in identifiable linear models \citep{arora2013practical, bing2020adaptive, bing2020fast}. For example, \cite{arora2013practical} defined anchor words that anchor the topics of documents, i.e. if a document contains an anchor word of a particular topic, then this document must be about this topic. In this paper, the anchor group of a factor means that only this factor reconstructs the variables in this group. Denote $\boldsymbol{B}=\left(\boldsymbol{\beta}_1,\dots, \boldsymbol{\beta}_C\right)$, for $C$ number of data categories, and $\boldsymbol{f}_t=\left(f_{t,1},\dots, f_{t,K}\right)$, we have the following theorem:


\begin{theorem}
    Suppose the following assumptions hold:
    \begin{enumerate}[label=(\arabic*)]
        \item The decoder follows \labelcref{equ:sparse_decoder} with $C$ number of data categories.
        \item Each factor has a known anchor group.
    \end{enumerate}
If we have two sets of decoder parameters and factors: $\{\boldsymbol{\theta}, \boldsymbol{B}, \boldsymbol{f}_t\}$ and $\{\boldsymbol{\theta}^*, \boldsymbol{B}^*, \boldsymbol{f}^*_t\}$, which yield the same reconstructions of $\hat{\boldsymbol{x}}_t$, for $t=1,\dots, T$, then the recovery of $f_{t,k}$ only depends on $f^*_{t,k}$ and parameters learned in the decoders, i.e. $f_{t,k}$ is identified up to element-wise transformations, $h_{k}(\cdot)$, for $k=1,\dots, K$.
\end{theorem}

\begin{proof}
For $k\in\{1,\dots,K\}$, suppose the $i$-th variable of $\boldsymbol{x}_t$, $\boldsymbol{x}_{t,i}$, is from the anchor group of the $k$-th factor, then $\beta_{c_i,k}\neq 0$ and $\beta_{c,k^\prime}=0$ for all $k^\prime\neq k$. We can simplify $g^d_{i,\boldsymbol{\theta}}\left(\boldsymbol{f}_t\odot {{\boldsymbol{\beta}_{c_i}}}\right)$ to $\tilde{g}^d_{i,\boldsymbol{\theta}}\left(f_{t,k}{{\beta_{c_i,k}}}\right)=\left(g^d_{i,L}\circ g^d_{L-1} \circ \dots \circ g^d_{2} \circ \tilde{g}^d_{1,k}\right)_{\boldsymbol{\theta}}\left(f_{t,k}{{\beta_{c_i,k}}}\right)$, where $\tilde{g}^d_{1,k}$ is the part of $g^d_{1}$ that is associated with the $k$-th factor, so $\tilde{g}^d_{i,\boldsymbol{\theta}}\left(f_{t,k}{{\beta_{c_i,k}}}\right)=\tilde{g}^d_{i,\boldsymbol{\theta}^*}\left(f^*_{t,k}{{\beta^*_{c_i,k}}}\right)$. If $\tilde{g}^d_{i,\boldsymbol{\theta}}(\cdot)$ is invertible, then $f_{t,k}={\beta^{-1}_{c_i,k}}\left((\tilde{g}^d_{i,\boldsymbol{\theta}})^{-1}\circ\tilde{g}^d_{i,\boldsymbol{\theta}^*}\right)\left(f^*_{t,k}\right)$. If $\tilde{g}^d_{i,\boldsymbol{\theta}}(\cdot)$ is not invertible, there exists a mapping $h_{k}(\cdot)$ such that $f_{t,k}=h_{k}(f^*_{t,k})$.
\end{proof}

The identifiability in Theorem 1 mitigates rotational invariance and any invariance involving transformations that require multiple factors. Even though the factors are semi-identifiable, which is weaker than the canonical one such that $f_{t,k}=f^*_{t,k}$, for $t=1,\dots,T$ and $k=1,\dots,K$, in practice, we find that the GS autoencoder effectively identifies most factors after standardization and sign switching. 

\begin{table}[!htbp]
\small
\centering
  \begin{tabular}{lll}
\toprule
$\tilde{g}^d_{i_k,\boldsymbol{\theta}}(\cdot)$ & $h_k(\cdot)$ & Examples\\
    \midrule
    Invertible &  Has a one-to-one closed form. & \tabitem Invertible neural networks, e.g. \cite{dinh2014nice}.\\
    & & \tabitem Neural networks without activation function. \\
    \vspace{1mm}
    \\
    Injective &  A one-to-one mapping. & Neural networks satisfying the following conditions:\\
    & & \tabitem The activation function is injective.\\
    & & \tabitem The weights in $g^d_{l}$ are full column rank matrices ($l=2,\dots, L-1$).\\
    & & \tabitem The weight in $g^d_{L,i}$, $\boldsymbol{W}^d_{L,i}$, are modified to 
    \(\displaystyle \begin{pmatrix}
    \boldsymbol{W}^d_{L,i} \\
    \boldsymbol{I}_{2:D^d_{L-1},D^d_{L-1}}
    \end{pmatrix} \)
    . \vspace{1mm}\vspace{1mm}\\
    Non-injective & A one-to-many mapping & \tabitem Neural netowrks with the ReLU.\\
    
    \bottomrule
\end{tabular}
\caption{Properties of $h_{k}(\cdot)$ based on different decoders. $\boldsymbol{I}_{2:D^d_{L-1},D^d_{L-1}}$ is the second to the last rows of an $D_{L-1}$-by-$D_{L-1}$ identity matrix.}
\label{tab:property}
\end{table}

Finally, we shed light on the decoder architecture and choice of activation functions by exploiting different properties of the element-wise transformations, $h_{k}(\cdot)$, across various decoders. In particular, we focus on $\tilde{g}^d_{i_k,\boldsymbol{\theta}}$ (see notations in the proof of Theorem 1), where the $i_k$-th variable is from the anchor group of the $k$-th factor, for $k=1,\dots,K$. Table \ref{tab:property} presents these properties with the corresponding decoders and examples. More description about each connection between the property and the example can be found in Appendix \ref{sec:property}. To select the architecture and activation function, we firstly eliminate the third case when the decoder is not injective, because it is difficult to recover $f_{t,k}$ from $f^*_{t,k}$ via a one-to-many mapping. Between the invertible and injective decoders, we select the latter due to its simplicity since an injective activation function and the modification mentioned in Table \ref{tab:property} are straightforward, then the only assumption we need is full column rank matrices. Although a one-to-one closed form of $h_{k}(\cdot)$ is appealing, the complexity of specifying the invertible architecture extends beyond the scope of our current study, which could be an extension of our work. In this paper, $h_{k}(\cdot)$ is a one-to-one mapping as long as $\tilde{g}^d_{i_k,\boldsymbol{\theta}}$ is injective, for $k=1,\dots,K$, but for consistency, we set the architecture and activation function to be the same for all decoders which reconstruct $\boldsymbol{x}_t$. We use 5-fold cross-validation to select the injective activation function between tanh($\cdot$) and Leaky ReLU,  $g\left(x\right)$ = $\text{max}(ax,x)$, where $a$ is a multiplier smaller than 1, because of their popularity in the deep learning literature. 

After training the GS autoencoder using gradient descent (see details in Section \ref{sec:Training Deep Learning Model}), we adopt a similar procedure to the one in \cite{klieber2024non} to approximate a linear transformation between factors and high-dimensional data. This procedure yields easy derivation of the impulse responses of $\boldsymbol{x}_t$ to the shocks.

\subsection{TVP-FAVAR} 
\label{sec:FA-TVP-VAR}
Constructing a non-linear model in the first step of the two-step procedure implies that the dynamics between the high-dimensional data and factors is non-linear, so it is natural to also model the evolution of factors as non-linear. We employ the TVP-VAR structure in \cite{primiceri2005time} to express the time variation in the second step of the two-step procedure. Specifically, \labelcref{equ:FAVAR} changes to:
\begin{equation}
    \left(\begin{aligned}
&\boldsymbol{f}_t\\
&\boldsymbol{y}_t 
\end{aligned} 
\right) = \boldsymbol{A}_{t,1} \left(
\begin{aligned}
&\boldsymbol{f}_{t-1}\\
&\boldsymbol{y}_{t-1} 
\end{aligned} 
\right) + \dots+\boldsymbol{A}_{t,P} \left(
\begin{aligned}
&\boldsymbol{f}_{t-P}\\
&\boldsymbol{y}_{t-P} 
\end{aligned} 
\right) +\boldsymbol{\eta}_t,\, \boldsymbol{\eta}_t\sim\mathcal{N}(\textbf{0},\boldsymbol\Omega_t),
\end{equation}
where $\Omega_t=\boldsymbol{H}_t\boldsymbol{S}_t\boldsymbol{S}^\prime_t\boldsymbol{H}_t^\prime$, $\boldsymbol{H}_t$ is a lower triangular matrix with ones on the diagonal and $\boldsymbol{S}_t$ is a diagonal matrix. 

We vectorize $\left(\boldsymbol{A}_{t,1},\dots, \boldsymbol{A}_{t,P}\right)$, non-zero unknown entries in $\boldsymbol{H}^{-1}_t$ and diagonal entries in $\boldsymbol{S}_t$ to $\boldsymbol{a}_t$, $\boldsymbol{h}_t$ and $\boldsymbol{s}_t$ respectively, and let them follow (log) random walks:\\
\begin{minipage}{0.45\textwidth}
\begin{align*}
    \boldsymbol{a}_t &= \boldsymbol{a}_{t-1} + \boldsymbol{\xi}_{a,t},\\
    \boldsymbol{h}_t &= \boldsymbol{h}_{t-1} + \boldsymbol{\xi}_{h,t},
    \\
    \log{\boldsymbol{s}_t} &= \log{\boldsymbol{s}_{t-1}} +  \boldsymbol{\xi}_{s,t},
\end{align*}
\end{minipage}
\begin{minipage}{0.45\textwidth}
\begin{equation*}
\begin{split}
    \left(
\begin{aligned}
&\boldsymbol{\xi}_{a,t}\\
&\boldsymbol{\xi}_{h,t}\\
&\boldsymbol{\xi}_{s,t}
\end{aligned} 
\right)& \sim \mathcal{N}\left(\mathbf{0},\begin{pmatrix}
\boldsymbol{Q}_a & 0 & 0\\
0 & \boldsymbol{Q}_h & 0\\
0 & 0 & \boldsymbol{Q}_s
\end{pmatrix}\right),
\end{split}
\end{equation*}
\end{minipage} \\
where $\boldsymbol{Q}_a$ and $\boldsymbol{Q}_s$ have no restrictions but positive definite matrices, and $\boldsymbol{Q}_h$ is block-diagonal so that the elements in $\boldsymbol{h}_t$ are only correlated to the elements in the same row of $\boldsymbol{H}^{-1}_t$.

Since we approximate a linear transformation between factors and the high-dimensional data, one may impose an additional time-varying structure to \labelcref{equ:factor model} to have $\boldsymbol{\Lambda}_t$ and $\boldsymbol{\Sigma}_t$. We do not include this structure mainly due to the high computational cost of the MCMC. For example, our real data application needs to infer 346,005 additional parameters if we use this structure.


\section{Estimation}
\label{sec:Estimation}

We split the inference of parameters into two steps according to the two-step procedure. The first step is to train parameters in the GS autoencoder, i.e. $\{\boldsymbol{\phi},\boldsymbol{\theta}, \boldsymbol{B}, \boldsymbol{\Gamma}\}$, where $\boldsymbol{\Gamma}=\left(\boldsymbol{\gamma}_1,\dots, \boldsymbol{\gamma}_C\right)$ with $\boldsymbol{\gamma}_c = (\gamma_{c,1},\dots, \gamma_{c,K})^\prime$, and the second step infers the parameters in the TVP-VAR, i.e. $\{\boldsymbol{a}_t, \boldsymbol{h}_t, \boldsymbol{s}_t, \boldsymbol{Q}_a, \boldsymbol{Q}_h, \boldsymbol{Q}_s\}$. The following two subsections provide details of these two steps, respectively.

\subsection{Training the Deep Learning Model} 
\label{sec:Training Deep Learning Model}
The training process is essentially to maximize the marginal loglikelihood $\log p\left(\boldsymbol{x}_{1:T}\right)$. Since this likelihood is intractable due to the non-linear activation function, we maximize an objective function that is an evidence lower bound (ELBO) of the likelihood instead. The objective function is written as:
\begin{align}
\mathcal{L}\left(\boldsymbol{\phi},\boldsymbol{\theta},\boldsymbol{B}\right)=&-\frac{1}{2T}\sum_{t=1}^T\text{MSE}\left(\boldsymbol{x}_t,\hat{\boldsymbol{x}}_{t}\right)+ \frac{1}{T}\frac{1}{N}\sum_{c=1}^C\sum_{k=1}^K \left[p_{c,k}\left(\log \psi_1\left(\beta_{c,k}\right)-\log p_{c,k}\right)\right.\\
    &\left.+\left(1-p_{c,k}\right)\left(\log \psi_0\left(\beta_{c,k}\right)-\log\left(1-p_{c,k}\right)\right)\right], \label{equ:ELBO}
\end{align}
where MSE denotes mean squared error, $p_{c,k}=\mathbb{E}[\gamma_{c,k}\mid \beta_{c,k}]=\frac{\psi_1\left(\beta_{c,k}\right)}{\psi_0\left(\beta_{c,k}\right)+\psi_1\left(\beta_{c,k}\right)}$. The derivation of this objective function is in Appendix \ref{sec:Derivation of the ELBO}.

This ELBO is composed of two parts. The mean squared error part guides the latent factors to form effective low-dimensional representations of $\boldsymbol{x}_{1:T}$. Note that the division of 2 in front of the summation assumes the variance of each variable in $\boldsymbol{x}_t$ is 1. The remaining part regularizes $\boldsymbol{B}$ to follow an SSL prior. Algorithm \ref{alg:1} summarizes the training of all parameters in the GS autoencoder. $\lambda_0$ and $\lambda_1$ are chosen by the cross-validation. $m$ presents the $m$-th iteration, which passes all data points in the training process. We consider the mini-batch gradient descent, which uses only a batch of data every time to update the parameters, and is known to be efficient and stable. The optimizer is Adaptive Moment Estimation (Adam) \citep{kingma2014adam}. We set the number of iterations, i.e. epochs, and the batch size as 200 and 24, respectively.

\begin{breakablealgorithm}
\caption{Training the GS autoencoder}\label{alg:1}
\begin{algorithmic}
\State \textbf{Input:} $\boldsymbol{x}_{1:T}$, $\lambda_0$ and $\lambda_1$.
\State \textbf{Output:} $\boldsymbol{\phi}$, $\boldsymbol{\theta}$, $\boldsymbol{B}$ and $\boldsymbol{\Gamma}$. 
\For {$m$ in $1,\dots,$ epochs}:
\For {each batch}:
\State Update $\boldsymbol{\phi}$, $\boldsymbol{\theta}$, $\boldsymbol{B}$ according to $\mathcal{L}\left(\boldsymbol{\phi},\boldsymbol{\theta},\boldsymbol{B}\right)$ with Adam.
\State Update the posterior of $\gamma_{c,k}$, for $c=1,\dots,C$ and $k=1,\dots,K$: \[p\left(\gamma_{c,k}\mid\beta_{c,k}\right)=\frac{\psi_1\left({\beta}_{c,k}\right)}{\psi_0\left({\beta}_{c,k}\right)+\psi_1\left({\beta}_{c,k}\right)}\].
\EndFor
\EndFor
\end{algorithmic}
\end{breakablealgorithm}

\subsection{Bayesian Inference} 
\label{sec:Bayesian Inference}
We use Bayesian inference to learn the parameters in the TVP-VAR. To induce parsimony, we follow \cite{korobilis2013assessing} to set a Minnesota-type prior to $\boldsymbol{a}_0$, the vectorization of $\boldsymbol{A}_0=\left(\boldsymbol{A}_{0,1},\dots, \boldsymbol{A}_{0,P}\right)$. Specifically, the $(i,j)$ entry of $\boldsymbol{A}_{0,p}$ follows $\mathcal{N}\left(0, \underline{\boldsymbol{V}}_{p,(i,j)}\right)$, with $
    \underline{V}_{p,(i,j)}=\begin{cases}
    \frac{0.7}{p^2}, & \text{if }i=j\\
    \frac{0.1}{p^2}\frac{\hat{\sigma}_i}{\hat{\sigma}_j}, &\text{if }i\neq j
    \end{cases}$, where $\hat{\sigma}^2_i$ is the variance estimate of $\boldsymbol{y}_{t,i}$ sequence modeled by an AR(2) process, the multipliers (0.7 and 0.1) mitigate explosive draws without sacrificing the time variation of the parameters. The priors of $\boldsymbol{h}_0$ and $\log\boldsymbol{s}_0$ are $\mathcal{N}\left(\mathbf{0}, 4\boldsymbol{I}\right)$. Denote the variance-covariance matrix of $\boldsymbol{a}_0$, $\boldsymbol{h}_{0,m}$ (non-zero entries on the $m$-th row of $\boldsymbol{H}^{-1}_0$ for $m=2,\dots, M+K$) and $\log\boldsymbol{s}_t$ as $\underline{\boldsymbol{V}}_a$, $\underline{\boldsymbol{V}}_{h,m}$ and $\underline{\boldsymbol{V}}_{s}$, $\boldsymbol{Q}_{a}$, $\boldsymbol{Q}_{h,m}$ and $\boldsymbol{Q}_s$ follow inverse-Wishart priors: $\boldsymbol{Q}_{a}\sim\mathcal{IW}\left(0.0001\times (\text{dim}(\boldsymbol{a}_0)+1)\times \underline{\boldsymbol{V}}_a, \text{dim}(\boldsymbol{a}_0)+1\right)$, $\boldsymbol{Q}_{h,m}\sim$$\mathcal{IW}\left(0.0001\times \text{dim}((\boldsymbol{h}_{0,m})+1)\times \underline{\boldsymbol{V}}_{h,m}, \text{dim}(\boldsymbol{h}_{0,m})+1\right)$, then $\boldsymbol{Q}_{s}\sim\mathcal{IW}\left(0.01\times (\text{dim}(\boldsymbol{s}_0)\right.$\\$\left.+1)\times \underline{\boldsymbol{V}}_s, \text{dim}(\boldsymbol{s}_0)+1\right)$\footnote{If an $n$-by-$n$ positive definite matrix $\boldsymbol{Q}$ follows $\mathcal{IW}\left(\boldsymbol{V},\nu\right)$, where $\boldsymbol{V}\in\mathbb{R}^{n\times n}$ and $\nu$ is a scalar, then its probability density function is $\frac{\lvert A \rvert ^{\nu/2}}{2^{n \nu/2}\Gamma_n(\nu/2)}\lvert \boldsymbol{Q}\rvert ^{-\frac{\nu+n+1}{2}}e^{-\frac{1}{2}\text{tr}\left(\boldsymbol{V}\boldsymbol{Q}^{-1}\right)}$, where $\Gamma_n$ is the multivariate gamma function.}, where dim$(\boldsymbol{a}_0)=p(M+K)^2$, dim$(\boldsymbol{h}_{0,m})=m-1$ and dim$(\boldsymbol{s}_0)=M+K$. Since we approximate a linear factor model after getting the factors from the deep learning model, we need to impose priors to $\boldsymbol{\Lambda}$ and $\boldsymbol{\Sigma}$. In particular, $\text{vec}(\boldsymbol{\Lambda})\sim\mathcal{N}\left(\mathbf{0}, 4\boldsymbol{I}\right)$ and $\boldsymbol{\Sigma}^{-1}_{i,i}\sim\text{Gamma}\left(0.01,0.01\right)$ with zero non-diagonal entries.

    The inference of the TVP-VAR adopts the MCMC algorithm proposed in \cite{del2015time}, which is the corrigendum of \cite{primiceri2005time} that alters the ordering of sampling blocks to yield draws from correct posterior and uses a Metropolis-Hastings step to infer $\boldsymbol{s}_t$. In particular, we sample $\boldsymbol{a}_t$ and $\boldsymbol{h}_t$ via the forward-filtering-backward-sampling algorithm \citep{fruhwirth1994data, carter1994gibbs}, and sample $\boldsymbol{s}_t$ with a mixture model following \citep{kim1998stochastic}. In the real data application, we use an R package called \textbf{bvarsv} \citep{krueger2015bvarsv} to sample these parameters with 10,000 iterations of burn-in and 100,000 iterations of MCMC sampling. More details about the sampling of the linear approximation, $\boldsymbol{\Lambda}$ and $\boldsymbol{\Sigma}$ , are in Appendix \ref{sec:Estimation Details}. 


\section{Empirical Results}
\label{sec:Empirical Results}

\subsection{Data and Implementation Detail}
\label{sec:Data and Implementation Detail}
We use 168 quarterly US macroeconomic variables in \cite{mccracken2020fred} to demonstrate the utility of the proposed model. The data ranges from 1965:Q1 to 2023:Q1, and is divided into 12 groups according to \cite{mccracken2020fred}: (1) National Income and Product Accounts (NIPA), (2) industrial production, (3) earning and productivity, (4) labor market, (5) housing, (6) inventory, orders and sales, (7) prices, (8) interest rate, (9) money and credit, (10) household balance sheets, (11) exchange rates and (12) stock market. All time series are transformed to stationarity and standardized, as is conventional in the FAVAR literature. Readers can refer to Appendix \ref{sec:Data} for more details about the data. For all FAVAR models considered in this application, the observable factors ($\boldsymbol{y}_t\in\mathbb{R}^3$) are the gross domestic product: implicit price deflator (GDPDEF), unemployment rate (UNRATE) and effective Federal funds rate (FEDFUNDS), which are the proxies of inflation, labor market and interest rate, respectively. The other variables construct the high-dimensional variable $\boldsymbol{x}_t\in\mathbb{R}^{165}$ that we extract factors from. 

There are 4 dimension reduction methods considered in this study: PCA, standard autoencoder, the GS autoencoders with either identity or non-linear activation functions. The former GS autoencoder is a linear model analogous to the structural FAVAR \citep{belviso2006structural}, and the latter is non-linear. We present the implementation of the non-linear GS autoencoder, as it is the most general case. Readers can adapt relevant components of this implementation for other dimension reduction methods of interest. In particular, the PCA only requires specifying the number of factors, the standard autoencoder implementation does not involve anchor groups or the SSL, and the linear GS autoencoder replaces the non-linear activation function with an identity function.

According to the cross-validation result, we extract 5 factors from the high-dimensional data, and find that the same number of principal components explain about 67\% of the variation. The anchor groups of these factors are: NIPA, labor market, prices, interest rates, money and credit. We choose these 5 groups because they align as closely as possible with those considered in the structural FAVAR of \cite{belviso2006structural}, and cover a large proportion of the variables. After setting the anchor groups, we assume that the first $5$ rows of $\boldsymbol{\Gamma}$ (the matrix corresponding to $\boldsymbol{B}$ to indicate the spike or slab lasso, see definition in Section \ref{sec:Training Deep Learning Model}) form an identity matrix. For the SSL, the cross-validation suggests that the hyperparameter choices are $\lambda_0=1000$ and $\lambda_1=1$. As discussed in \cite{rovckova2016fast}, setting a very large value to $\lambda_0$ (like in this case) allows $\lambda_0\approx\infty$ in practice, which leads to a Dirac spike at zero in the SSL. Thus, this hyperparameter choice further strengthens our assumption about anchor groups because the top 5 rows of $\boldsymbol{B}$ will be approximately diagonal.

Then we turn to the remaining architecture of the encoder, $g^e_{\boldsymbol{\phi}},$ and decoder, $g^d_{i,\boldsymbol{\theta}}$, for $i=1,\dots,N$. Three sets of hyperparameters need to be determined: the number of layers ($L$), the dimensions of layers in the encoder ($D^e_l$ for $l=1,\dots, L$) and those in the decoder ($D^d_l$ for $l=1,\dots, L$). The cross-validation chooses $L=3$, then we can evenly downsize the dimensions from $N=165$ to $K=5$ and get $\left(D^e_1, D^e_2, D^e_3\right)=(111, 58, 5)$. For the last set of hyperparameters, we adopt a mirror structure of the encoder: $\left(D^d_1, D^d_2, D^d_3\right)=(58, 111,1)$, where the last dimension is 1 because the decoder, $g^d_{i,\boldsymbol{\theta}}$, reconstructs the $i$-th variable. Lastly, we select the activation function as the Leaky ReLU with $a=10^{-16}$ from the cross-validation result if the GS autoencoder is non-linear. 

For the evolution of factors, we compare the time-invariant (TIV) and TVP model specifications. We set the lag order $(P)$ to be 2, which is the same as that in \cite{korobilis2013assessing}. The TIV one imposes a Minnesota-type prior to the coefficient matrix and an inverse-Wishart prior to the variance-covariance matrix. Appendix \ref{sec:Bayesian Inference of Time-invariant VAR} provides more details about the prior setting and the Bayesian inference.  

\subsection{Analysis of Latent Factors}
\label{sec:Analysis of Latent Factors}

This subsection compares factors extracted via PCA and the non-linear GS autoencoder. While comparing the non-linear GS autoencoder and the standard autoencoder would be natural, we exclude the latter from this analysis due to its non-identifiable factors. Given that PCA factors are identifiable under restrictions mentioned in Section \ref{sec:Linear FAVAR} and share similar interpretability limitations discussed in Section \ref{sec:Challenges of Standard Autoencoder and Solutions}, they serve as suitable alternatives to the standard autoencoder factors for this comparison. The comparison of factors from the linear and non-linear GS autoencoders is in Appendix \ref{sec:Additional Results of Real Data Application}. To facilitate comparison, we permute the PCA factors to maximize their correlation with their respective GS autoencoder counterparts. These permuted factors explain about 22\%, 25\%, 7\%, 4\%, and 8\% of the data variation, respectively. 

A conventional approach to interpreting FAVAR latent factors involves plotting factor time series alongside the variable in $\boldsymbol{x}_t$ exhibiting the highest correlation with each factor. However, this method overlooks other variables that also demonstrate strong correlations with the factors. Thus, we follow \cite{klieber2024non} to record the variables with the highest 15 correlation magnitudes to each factor. Figure \ref{fig:factor_ts_nonlinesr_1} - \ref{fig:factor_ts_nonlinesr_5} present the time series of both PCA and the non-linear GS autoencoder factors, as well as the magnitudes of correlations between these factors and the 15 variables. The factors extracted from these dimension reduction methods show varying degrees of similarity. The first and second factors exhibit stronger similarity, with 7 and 11 common variables, respectively. While the third factors share 3 variables in common, the fourth and fifth factors show no overlap in their recorded variables. 

The first factors from both methods capture recession periods, with the GS autoencoder factor showing stronger sensitivity to recent crises like the dot-com bubble, GFC, and COVID-19 pandemic. While the PCA factor correlates with various economic indicators (prices, labor market, NIPA, and industrial production) representing broad real activities, the GS autoencoder factor exhibits stronger correlations specifically with NIPA and industrial production variables, suggesting a more focused representation of these categories.

\begin{figure}[!htbp]
\centering
\includegraphics[width=\textwidth]{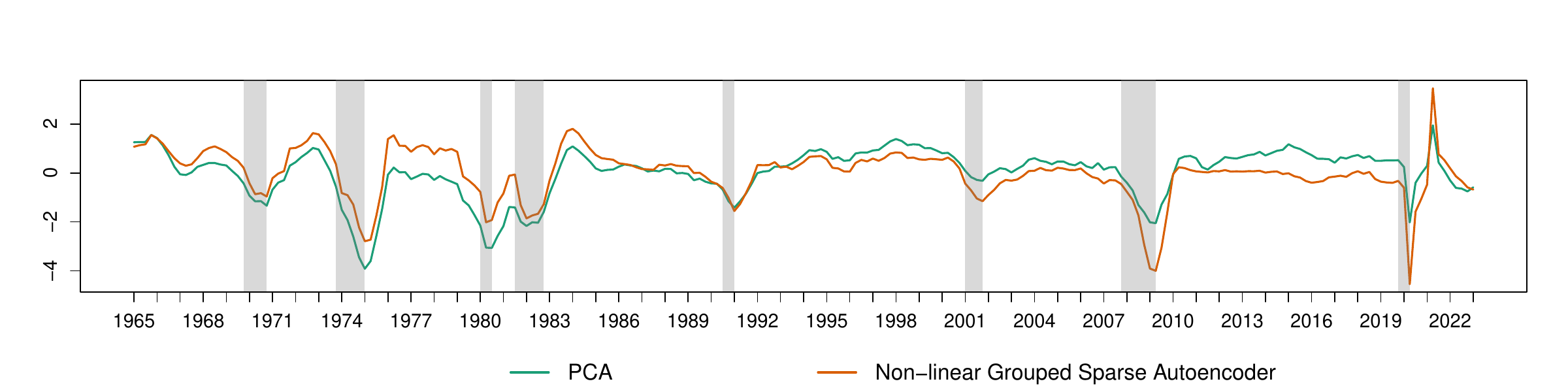}
\vskip 0.5cm
\includegraphics[width=0.8\textwidth]{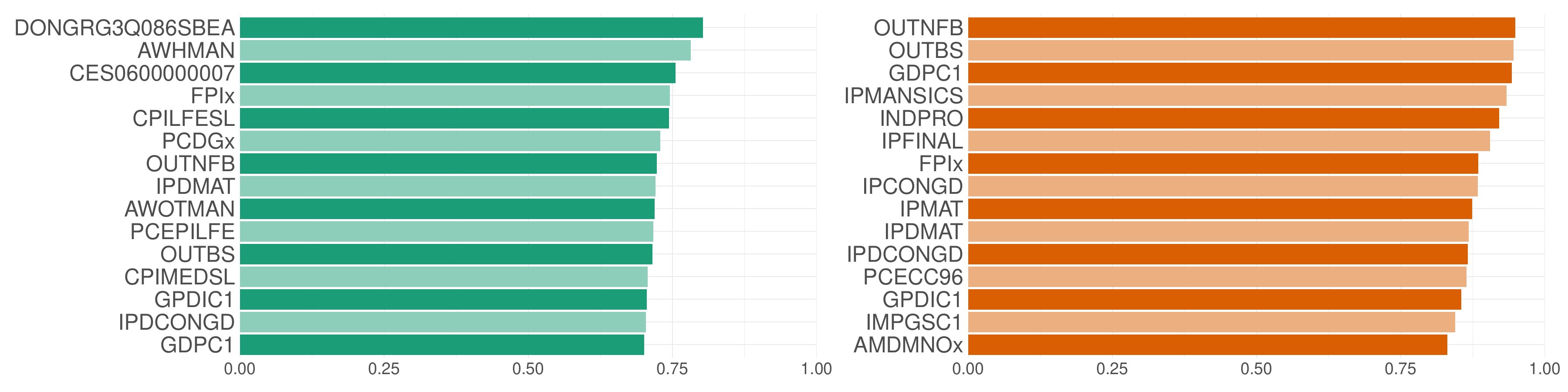}
\vskip\floatsep
\caption{The first factor extracted from the PCA and non-linear GS autoencoder (top panel), and variables with the 15 highest correlation magnitudes with the corresponding factors (bottom panel). The time series are standardized to have zero mean and variance one. The grey bands highlight the recession periods.}
\label{fig:factor_ts_nonlinesr_1}
\end{figure}

The second factors from the two methods have strong correlations with labor market data. Given that both factors show pronounced spikes during recession periods, this factor can be interpreted as a measure of labor market distress. The non-linear factor is smoother than its PCA counterpart, especially post 1990s. 

\begin{figure}[!htbp]
\centering
\includegraphics[width=\textwidth]{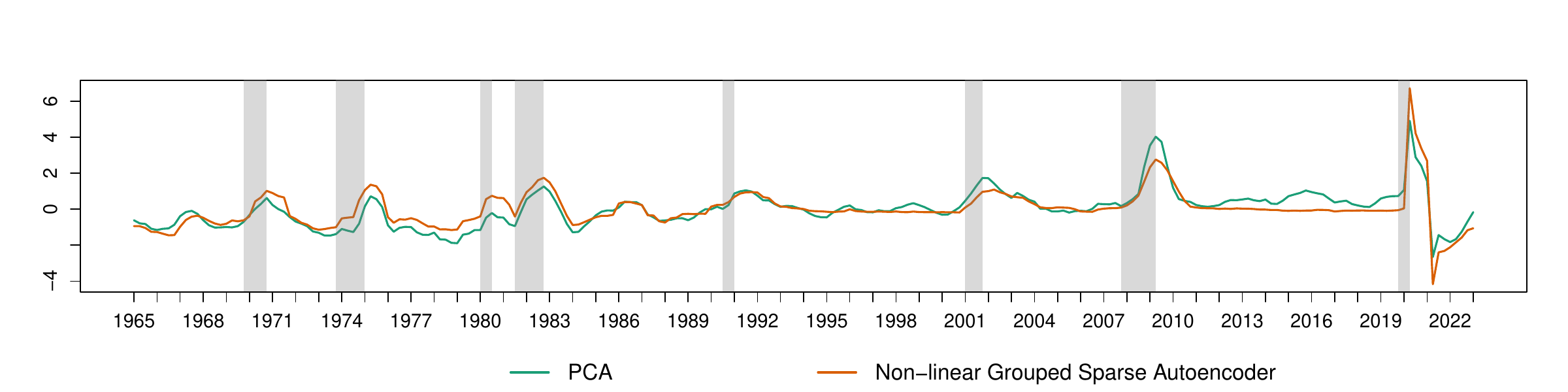}
\vskip 0.5cm
\includegraphics[width=0.8\textwidth]{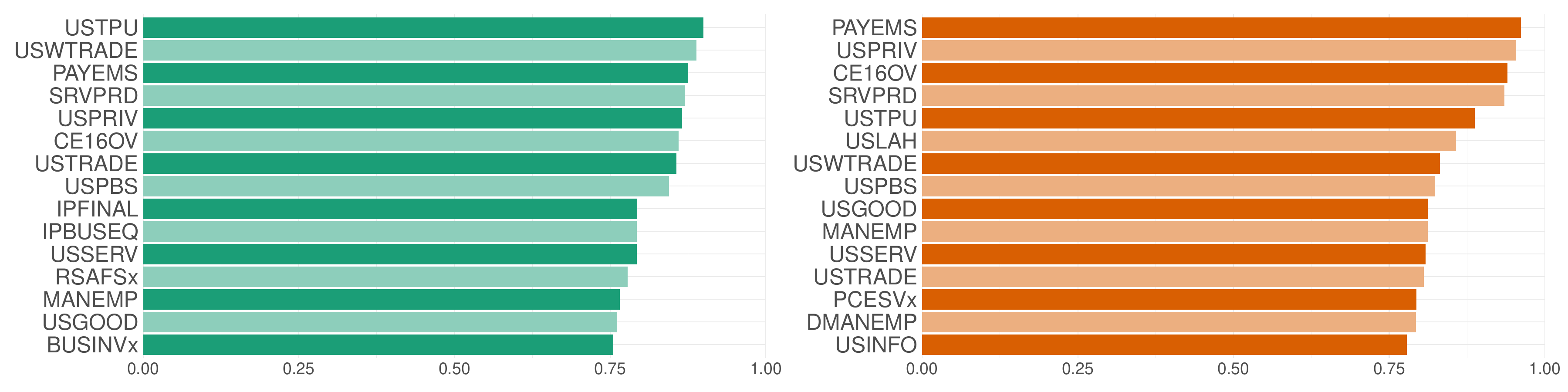}
\vskip\floatsep
\caption{The second factor extracted from the PCA and non-linear GS autoencoder (top panel), and variables with the 15 highest correlation magnitudes with the corresponding factors (bottom panel). The time series are standardized to have zero mean and variance one. The grey bands highlight the recession periods.}
\label{fig:factor_ts_nonlinesr_2}
\end{figure}

The third factors from both methods relate to prices, but differ in their focus and correlation magnitudes. The GS autoencoder factor emphasizes consumption prices, while the PCA factor captures both consumption and producer prices. However, the GS autoencoder factor shows a consistently stronger correlation (>0.75) with its price variables compared to the PCA factor, where only one correlation exceeds 0.7.

\begin{figure}[!htbp]
\centering
\includegraphics[width=\textwidth]{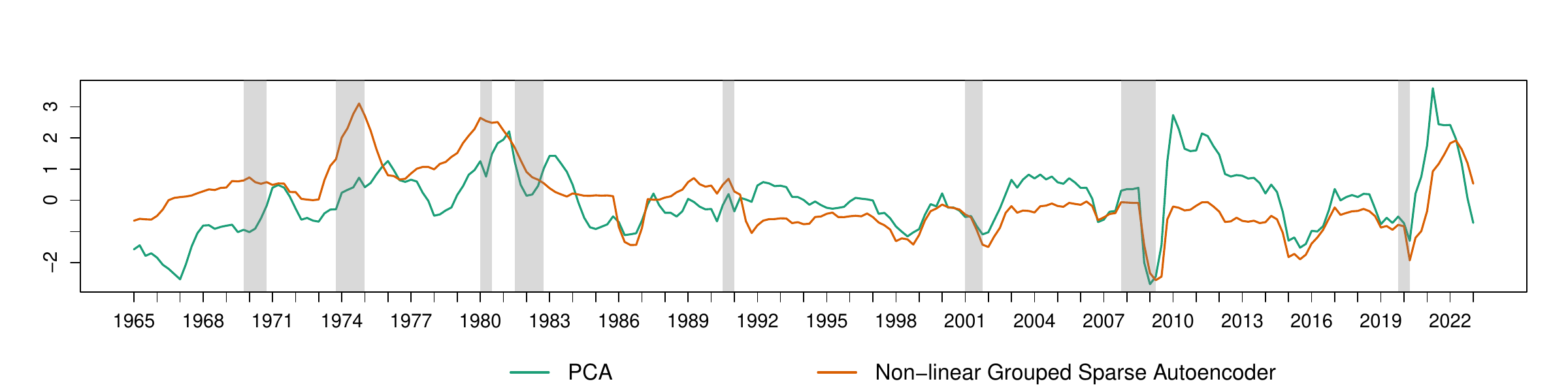}
\vskip 0.5cm
\includegraphics[width=0.8\textwidth]{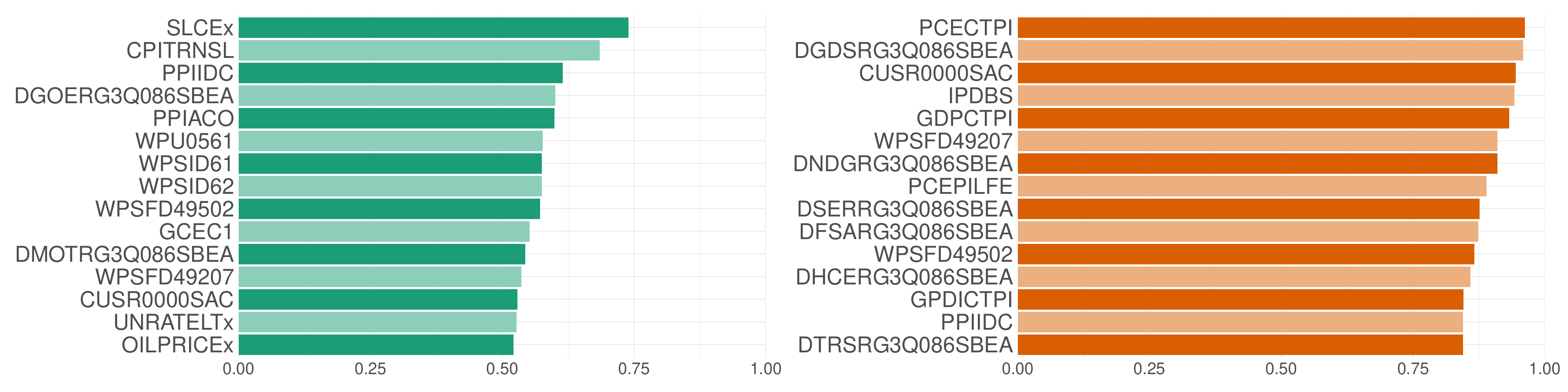}
\vskip\floatsep
\caption{The third factor extracted from the PCA and non-linear GS autoencoder (top panel), and variables with the 15 highest correlation magnitudes with the corresponding factors (bottom panel). The time series are standardized to have zero mean and variance one. The grey bands highlight the recession periods.}
\label{fig:factor_ts_nonlinesr_3}
\end{figure}

\begin{figure}[!htbp]
\centering
\includegraphics[width=\textwidth]{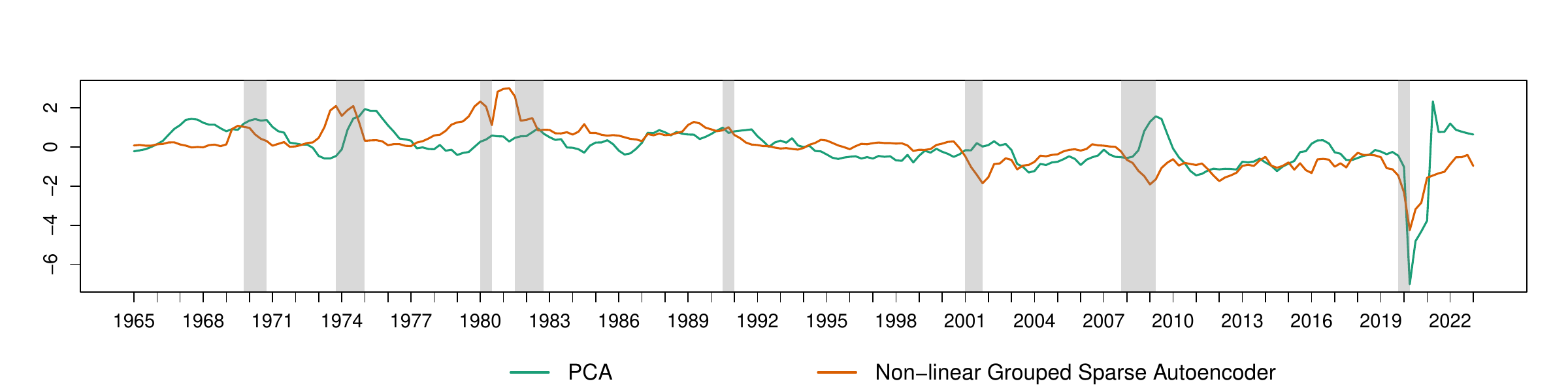}
\vskip 0.5cm
\includegraphics[width=0.8\textwidth]{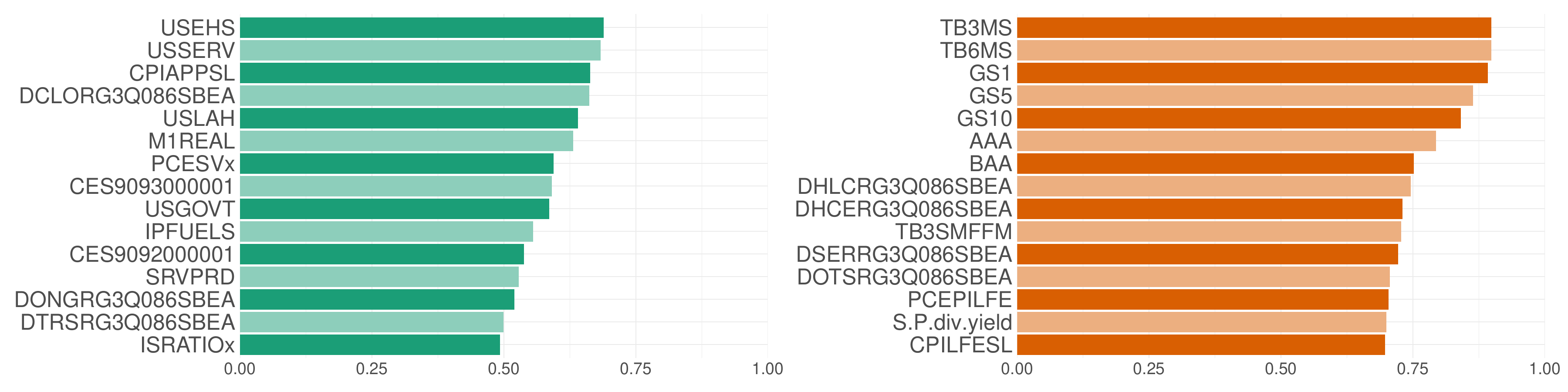}
\vskip\floatsep
\caption{The fourth factor extracted from the PCA and non-linear GS autoencoder (top panel), and variables with the 15 highest correlation magnitudes with the corresponding factors (bottom panel). The time series are standardized to have zero mean and variance one. The grey bands highlight the recession periods.}
\label{fig:factor_ts_nonlinesr_4}
\end{figure}

The fourth factor extracted from the GS autoencoder clearly represents interest rates, as it captures the major monetary decisions of the Federal Reserve. This factor is strongly correlated with short and long-term interest rates as well as the price variables closely monitored by the Federal Reserve. Unlike the concentration of the GS autoencoder factor, the PCA factor exhibits broad correlations across labor market and price variables (similar to its second and third factors), making its economic interpretation less clear.

Analysis of the fifth factors suggests that they had a similar trend before 1995 and then turned out to be negatively correlated. The two factors represent different effects, as almost half of the variables corresponding to the PCA are about housing, while the non-linear factor is the only factor that reconstructs the money and credit variables, so it emphasizes more on them and those variables known to be related to this category, such as government and corporate yields and prices. 

\begin{figure}[!htbp]
\centering
\includegraphics[width=\textwidth]{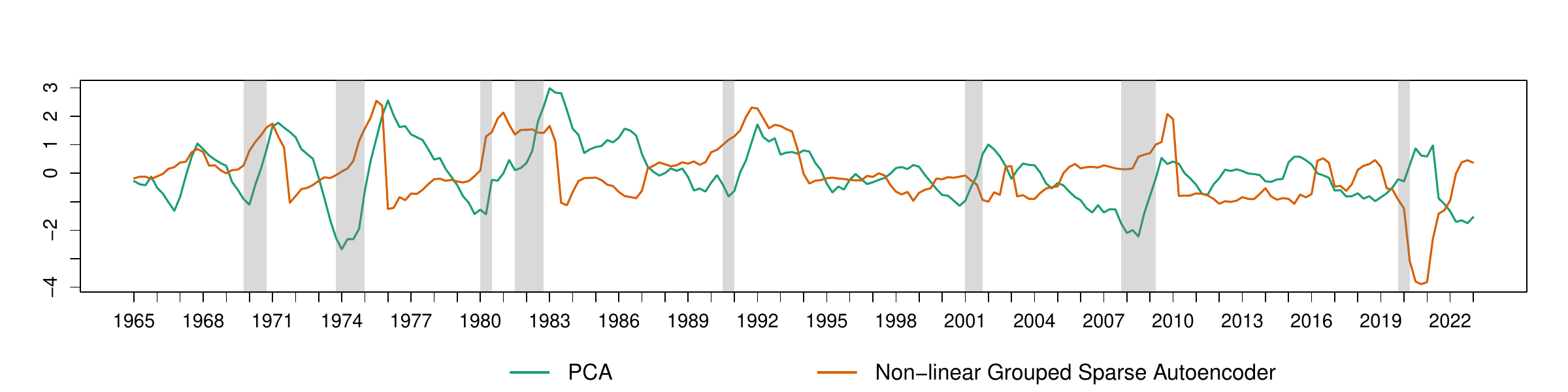}
\vskip 0.5cm
\includegraphics[width=0.8\textwidth]{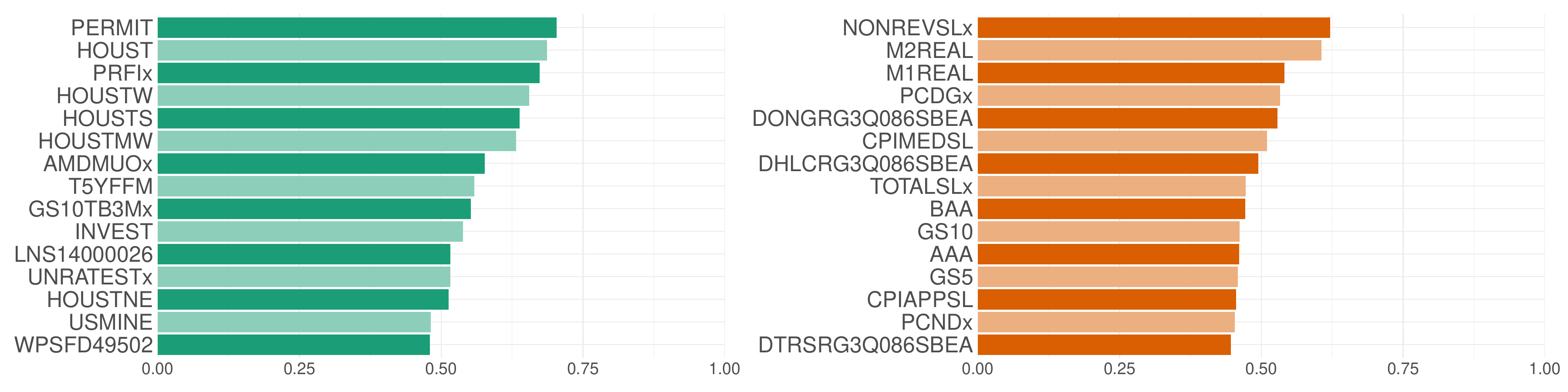}
\vskip\floatsep
\caption{The fifth factor extracted from the PCA and non-linear GS autoencoder (top panel), and variables with the 15 highest correlation magnitudes with the corresponding factors (bottom panel). The time series are standardized to have zero mean and variance one. The grey bands highlight the recession periods.}
\label{fig:factor_ts_nonlinesr_5}
\end{figure}

\begin{figure}[!htbp]
     \centering
     \begin{subfigure}[b]{0.48\textwidth}
         \centering
         \includegraphics[width=\textwidth]{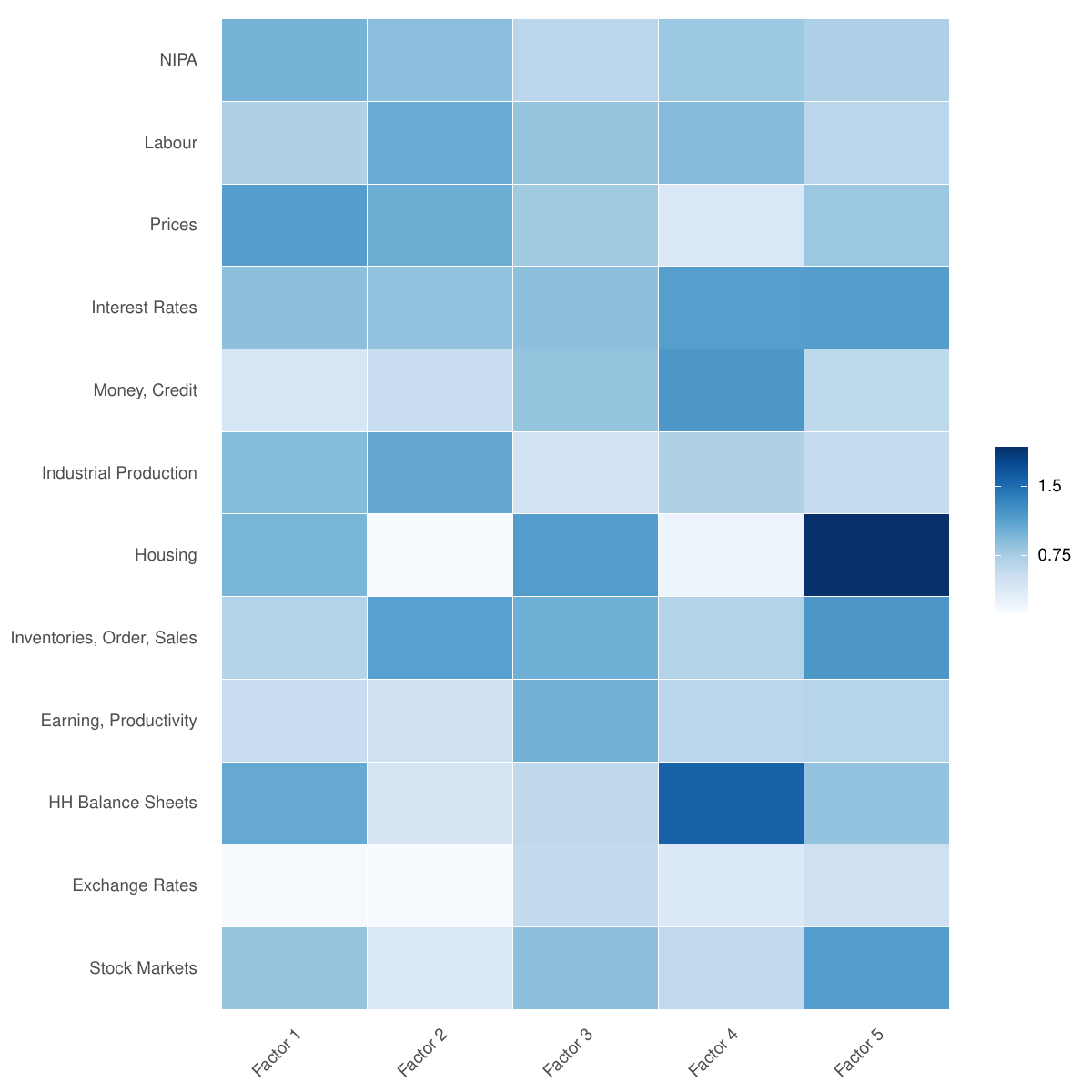}
         \caption{PCA}
         \label{fig: PCA factor importance}
     \end{subfigure}
     \hfill
     \begin{subfigure}[b]{0.48\textwidth}
         \centering
         \includegraphics[width=\textwidth]{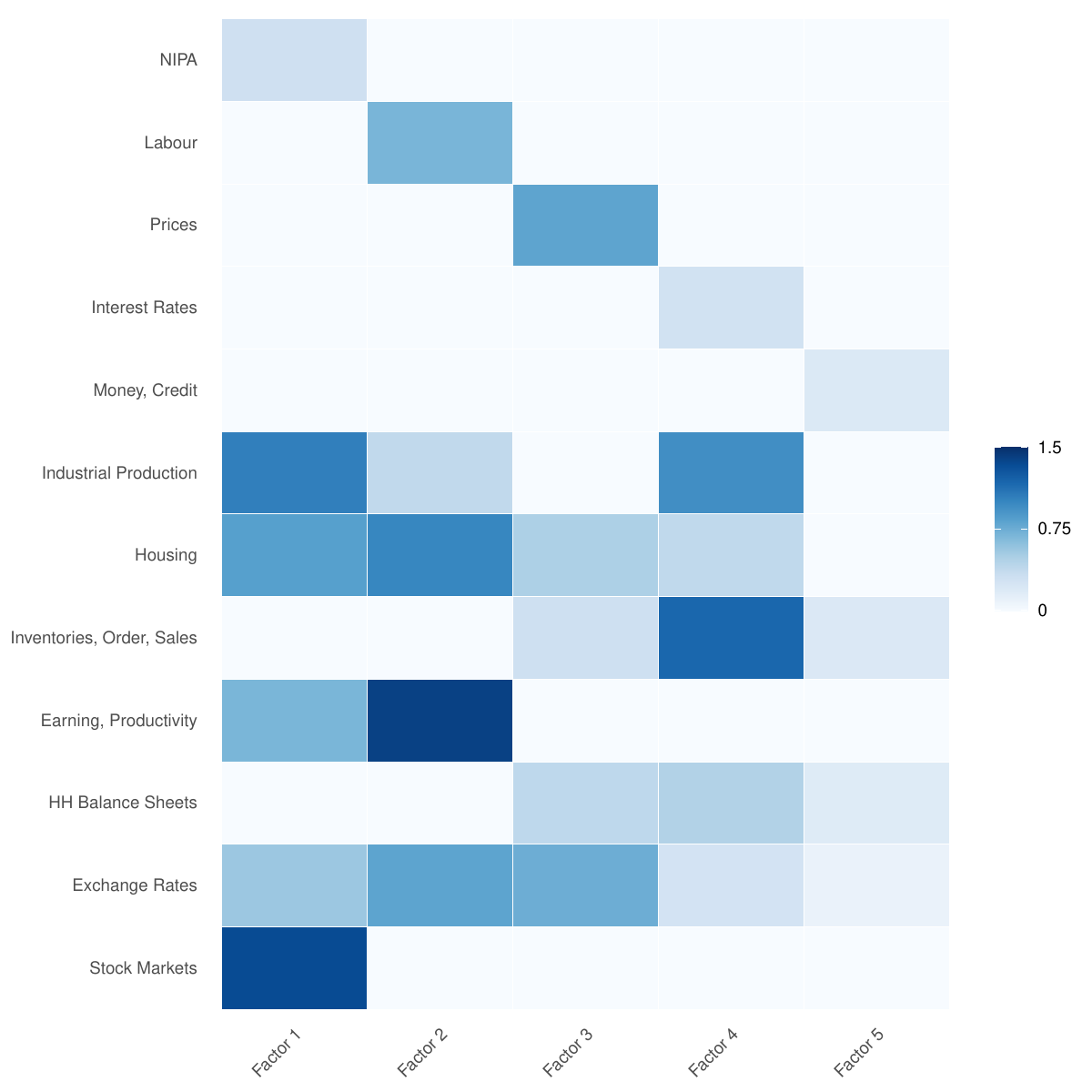}
         \caption{Non-linear Grouped sparse autoencoder}
         \label{fig: nonlinear grouped sparse ae}
     \end{subfigure}
        \caption{Importance of factors to different categories. "HH Balance Sheets" means houseshold balance sheets. Factors in the left panel is re-ordered so that each factor has a high correlation with the corresponding one in the right panel.}
        \label{fig: factor importance PCA and nonlinear grouped sparse ae}
\end{figure}

While Figures \ref{fig:factor_ts_nonlinesr_1} - \ref{fig:factor_ts_nonlinesr_5} demonstrate that GS autoencoder factors exhibit stronger group-specific correlations than PCA factors, some factors (particularly the first and second ones) from both methods are highly correlated. However, despite these similarities, the GS autoencoder provides more interpretable relationships between these factors and the high-dimensional data. Figure \ref{fig: factor importance PCA and nonlinear grouped sparse ae} depicts the importance of factors to data categories. The importance measure on the left panel is the averaged PCA loading over these categories, and the right panel uses $\boldsymbol{B}$, the SSL parameters. Overall, the GS autoencoder heat map is sparser than the PCA one, indicating better interpretability from a parsimonious structure. For PCA, the fourth and fifth factors can be primarily attributed to household balance sheets and housing, respectively, but identifying the main drivers of the first three factors is more challenging due to the comparable scales of their importance measures. In contrast, we do not have this issue in the right panel since the first five categories are anchor groups. Thus, we can name these factors according to their anchor groups. For instance, the first factor is called the NIPA factor. There is also a sparser structure among the non-anchor groups in the right panel, so it is easier to determine the factors that reconstruct these categories. For example, industrial production variables are mainly driven by the NIPA and interest rate factors; the labor market factor is the driving force for reconstructing earning and productivity variables.

\FloatBarrier

\subsection{Forecasting Performance}
\label{sec:Forecasting Performance}

We compare the forecasting performance of the FAVARs with 4 dimension reduction methods and 2 VAR specifications: time-invariant (TIV) and time-varying parameters (TVP). The inclusion of the standard autoencoder demonstrates the potential degradation of forecasting power due to non-identifiable factors, and the linear GS autoencoder serves as an approximation of models considered in \cite{belviso2006structural} and \cite{korobilis2013assessing}.

We use the expanding window procedure to make forecasts. In particular, we first fit a factor extraction model and the VAR model with the data from 1965:Q1 to 1983:Q4, then conduct the 1- to 4-step-ahead point and density forecasts in 1984. We repeat this procedure by adding one more data point to the training set each time until getting the 1-step-ahead forecasts in 2023:Q1.

Table \ref{tab:forecasting} presents the forecasting performance of different combinations of dimension reduction methods and model specifications. We use mean absolute error (MAE) and averaged log predictive likelihood (ALPL) to assess the point and density forecasts. We take the TIV-PCA as the benchmark model,
with its performance highlighted in grey, and all other evaluations are relative to the benchmark ones. The relative MAE is the ratio between the MAE of a model and the benchmark, so a value smaller than 1 indicates the superior point forecasting compared to the benchmark. Similarly, the relative ALPL is the difference between the ALPL of the model and that of the benchmark, so a value greater than 0 means the model is better in density forecasting.

\begin{table}[!htbp]
\footnotesize
\centering
\begin{tabular}{lcccclcccc}
\hline

Forecast   metric       & \multicolumn{1}{l}{MAE}       & \multicolumn{1}{l}{}          & \multicolumn{1}{l}{}          & \multicolumn{1}{l}{}          &  & \multicolumn{1}{l}{ALPL}       & \multicolumn{1}{l}{}           & \multicolumn{1}{l}{}           & \multicolumn{1}{l}{}           \\
\cline{2-5} \cline{7-10}
                        & \multicolumn{1}{l}{h=1}       & \multicolumn{1}{l}{h=2}       & \multicolumn{1}{l}{h=3}       & \multicolumn{1}{l}{h=4}       &  & \multicolumn{1}{l}{h=1}        & \multicolumn{1}{l}{h=2}        & \multicolumn{1}{l}{h=3}        & \multicolumn{1}{l}{h=4}        \\
                        \hline \rule{0pt}{1\normalbaselineskip}
                        & \multicolumn{4}{l}{GDPDEF}    &  & \multicolumn{1}{l}{}           & \multicolumn{1}{l}{}           & \multicolumn{1}{l}{}           & \multicolumn{1}{l}{}           \\
                        
TIV-PCA                 & \cellcolor[HTML]{D0D0D0}0.121 & \cellcolor[HTML]{D0D0D0}0.224 & \cellcolor[HTML]{D0D0D0}0.327 & \cellcolor[HTML]{D0D0D0}0.422 &  & \cellcolor[HTML]{D0D0D0}0.431  & \cellcolor[HTML]{D0D0D0}-0.489 & \cellcolor[HTML]{D0D0D0}-1.122 & \cellcolor[HTML]{D0D0D0}-1.992 \\
TVP-PCA                 & 0.821                         & 0.753                         & 0.716                         & 0.708                         &  & 0.200                          & 0.688                          & 1.044                          & 1.658                          \\
TIV-AE                  & 1.069                         & 1.093                         & 1.081                         & 1.064                         &  & -0.099                         & -0.110                         & -0.284                         & -0.254                         \\
TVP-AE                  & 0.859                         & 0.802                         & 0.759                         & 0.745                         &  & 0.176                          & 0.645                          & 0.999                          & 1.638                          \\
TIV-Linear GS AE    & 0.941                         & 0.973                         & 0.980                         & 1.013                         &  & 0.037                          & 0.105                          & 0.174                          & 0.669                          \\
TVP-Linear GS AE    & 0.795                         & 0.739                         & 0.704                         & 0.706                         &  & \textbf{0.218}                 & 0.709                          & 1.077                          & \textbf{1.701}                 \\
TIV-Nonlinear GS AE & 0.945                         & 0.972                         & 0.972                         & 0.998                         &  & 0.059                          & 0.122                          & 0.265                          & 0.436                          \\
TVP-Nonlinear GS AE & \textbf{0.785}                & \textbf{0.729}                & \textbf{0.691}                & \textbf{0.694}                &  & \textbf{0.218}                 & \textbf{0.710}                 & \textbf{1.079}                 & \textbf{1.701}                 \vspace*{-2mm} \\
                        & \multicolumn{1}{l}{}          & \multicolumn{1}{l}{}          & \multicolumn{1}{l}{}          & \multicolumn{1}{l}{}          &  & \multicolumn{1}{l}{}           & \multicolumn{1}{l}{}           & \multicolumn{1}{l}{}           & \multicolumn{1}{l}{}           \\
                        & \multicolumn{4}{l}{UNRATE}    &  & \multicolumn{1}{l}{}           & \multicolumn{1}{l}{}           & \multicolumn{1}{l}{}           & \multicolumn{1}{l}{}           \\
TIV-PCA                 & \cellcolor[HTML]{D0D0D0}0.175 & \cellcolor[HTML]{D0D0D0}0.280 & \cellcolor[HTML]{D0D0D0}0.375 & \cellcolor[HTML]{D0D0D0}0.461 &  & \cellcolor[HTML]{D0D0D0}-3.279 & \cellcolor[HTML]{D0D0D0}-5.616 & \cellcolor[HTML]{D0D0D0}-6.383 & \cellcolor[HTML]{D0D0D0}-7.090 \\
TVP-PCA                 & 0.811                         & 0.859                         & 0.846                         & 0.865                         &  & 3.501                          & \textbf{5.561}                 & \textbf{5.963}                 & \textbf{6.385}                 \\
TIV-AE                  & 0.949                         & 0.938                         & 0.943                         & 0.960                         &  & -0.211                         & 0.417                          & 1.317                          & 1.044                          \\
TVP-AE                  & \textbf{0.807}                & 0.866                         & 0.849                         & 0.868                         &  & \textbf{3.716}                 & 5.511                          & 5.945                          & 6.371                          \\
TIV-Linear GS AE    & 1.031                         & 1.008                         & 0.983                         & 0.958                         &  & 0.091                          & 0.349                          & 0.133                          & 0.947                          \\
TVP-Linear GS AE    & 0.819                         & 0.855                         & 0.847                         & 0.864                         &  & 3.632                          & 5.292                          & 5.895                          & 6.327                          \\
TIV-Nonlinear GS AE & 0.973                         & 0.961                         & 0.954                         & 0.924                         &  & 0.433                          & 0.635                          & 0.763                          & 0.808                          \\
TVP-Nonlinear GS AE & 0.817                         & \textbf{0.853}                & \textbf{0.840}                & \textbf{0.856}                &  & 3.645                          & 5.486                          & 5.874                          & 6.335                         \vspace*{-2mm}  \\
                        & \multicolumn{1}{l}{}          & \multicolumn{1}{l}{}          & \multicolumn{1}{l}{}          & \multicolumn{1}{l}{}          &  & \multicolumn{1}{l}{}           & \multicolumn{1}{l}{}           & \multicolumn{1}{l}{}           & \multicolumn{1}{l}{}           \\
                        & \multicolumn{4}{l}{FEDFUNDS}   &  & \multicolumn{1}{l}{}           & \multicolumn{1}{l}{}           & \multicolumn{1}{l}{}           & \multicolumn{1}{l}{}           \\
TIV-PCA                 & \cellcolor[HTML]{D0D0D0}0.125 & \cellcolor[HTML]{D0D0D0}0.228 & \cellcolor[HTML]{D0D0D0}0.313 & \cellcolor[HTML]{D0D0D0}0.379 &  & \cellcolor[HTML]{D0D0D0}0.195  & \cellcolor[HTML]{D0D0D0}-0.142 & \cellcolor[HTML]{D0D0D0}-0.434 & \cellcolor[HTML]{D0D0D0}-0.655 \\
TVP-PCA                 & 0.623                         & \textbf{0.659}                & \textbf{0.680}                & 0.722                         &  & 0.265                          & 0.118                          & 0.098                          & 0.084                          \\
TIV-AE                  & 1.015                         & 1.056                         & 1.066                         & 1.069                         &  & -0.066                         & -0.094                         & -0.075                         & -0.061                         \\
TVP-AE                  & 0.667                         & 0.738                         & 0.760                         & 0.791                         &  & 0.243                          & 0.084                          & 0.066                          & 0.060                          \\
TIV-Linear GS AE    & 0.831                         & 0.855                         & 0.865                         & 0.859                         &  & -0.008                         & 0.046                          & 0.113                          & \textbf{0.155}                 \\
TVP-Linear GS AE    & 0.621                         & 0.679                         & 0.703                         & 0.737                         &  & \textbf{0.288}                 & 0.139                          & 0.119                          & 0.112                          \\
TIV-Nonlinear GS AE & 0.894                         & 0.896                         & 0.910                         & 0.911                         &  & -0.020                         & 0.032                          & 0.083                          & 0.117                          \\
TVP-Nonlinear GS AE & \textbf{0.617}                & 0.670                         & 0.692                         & \textbf{0.717} & &0.285	&\textbf{0.143}	&\textbf{0.121}	&0.113\\ \hline
\end{tabular}
\caption{Point and density forecasting performance evaluated by the MAE and ALPL. Values highlighted in grey present the actual MAE and ALPL of the benchmark model, TIV-PCA, and the rest of the values are relative to those of the benchmark model. AE means autoencoder, and GS means grouped sparse. The best-performed model in each horizon and variable has its evaluation in bold.}
\label{tab:forecasting}
\end{table}

For the point forecasts, 8 out of 12 cases show that the non-linear GS autoencoder has the best performance, and the TVP outperforms the TIV in all the cases. For density forecasting, half of the evaluations indicate the superior performance of the non-linear GS autoencoder with the TVP. Sparsity yields notable improvements, especially in the density forecasts of GDPDEF and FEDFUNDS, with all evaluations for these two variables showing its advantage. The evaluations of UNRATE forecasts reveal a discrepancy between point and density forecasting performance. The non-linear GS autoencoder with heteroskedasticity predominantly outperforms other models as measured by the MAE, whereas the TVP-PCA model excels in density forecasting as indicated by the ALPL.

\begin{figure}[!htbp]
    \centering
    \includegraphics[width=\linewidth]{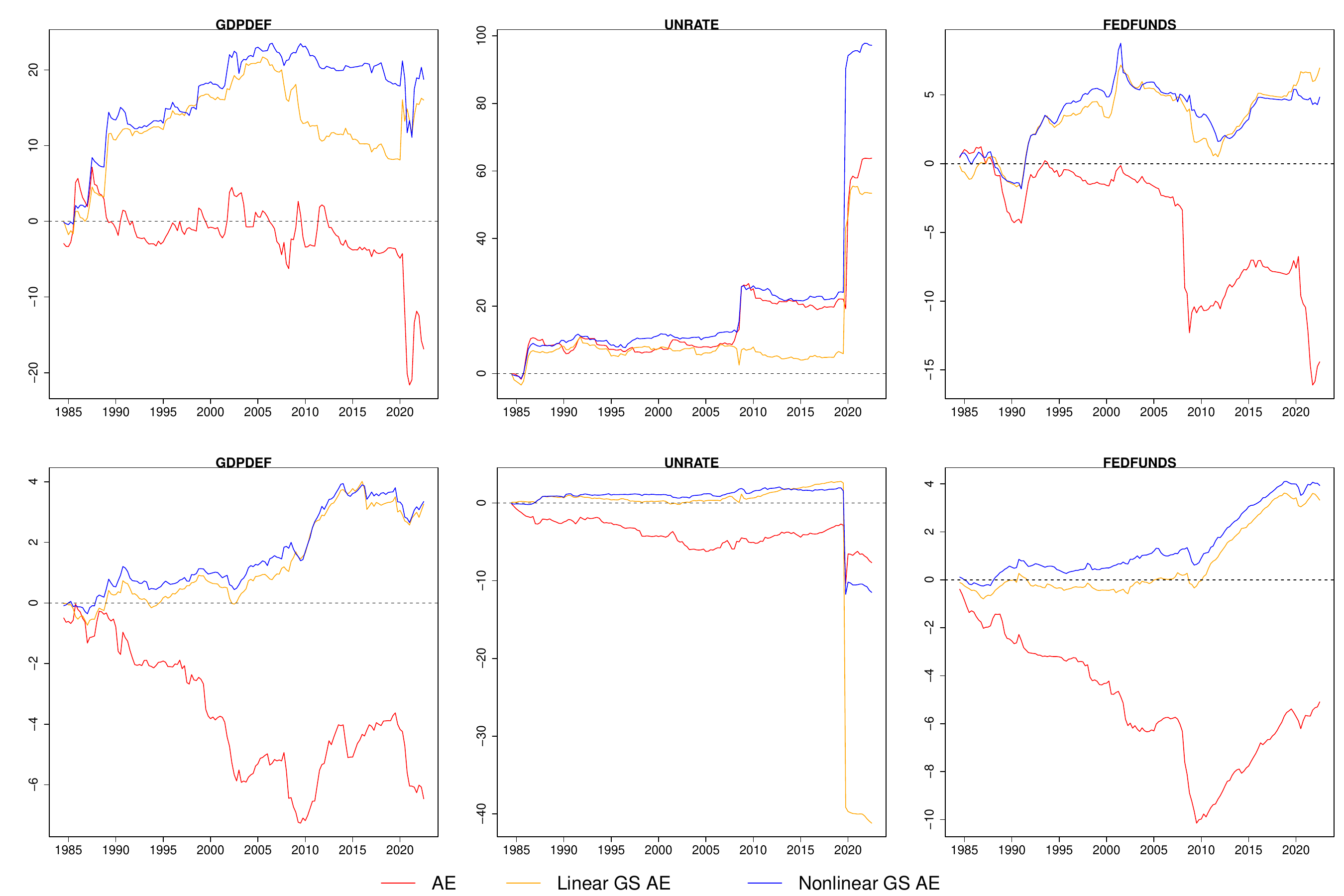}
    \caption{Cumulative ALPL (h=2) of models relative to the TIV-PCA (top) and the TVP-PCA (bottom).}
    \label{fig:cumulative_ALPL_h=2}
\end{figure}

While Table \ref{tab:forecasting} provides an overview of the forecasting performance, Figure \ref{fig:cumulative_ALPL_h=2} gives a more detailed analysis through the cumulative ALPL. This figure depicts the cumulative ALPLs of models with TIV parameters relative to the TIV-PCA and those with TVP relative to the TVP-PCA. A curve above the zero horizontal line means the performance of the associated model is better than its PCA counterpart. The red curves corresponding to the standard autoencoder in most panels are below zero, suggesting the importance of sparsity to the downstream forecasting task. The top panels about the TIV models show that the non-linearity from the autoencoder improves the forecasts of GDPDEF and UNRATE. In particular, the curves corresponding to the linear (orange) and non-linear (blue) GS autoencoders start to deviate around the GFC, due to the superior performance of the non-linear model. The curves about the FEDFUNDS do not exhibit such discrepancy. Moving to the TVP models, the performance of the non-linear GS autoencoder is slightly better than its linear counterpart in forecasting GDPDEF, FEDFUNDS and UNRATE before the COVID-19 pandemic. The less significant discrepancy between the blue and orange curves indicates that the introduction of the TVP structure partially diminishes the contribution of deep learning to the overall model performance. Following the onset of the COVID-19 pandemic, model performance about the UNRATE deteriorates relative to the PCA benchmark, but the non-linear GS autoencoder yields a less pronounced decline compared to its linear counterpart. To conclude, we find that both non-linearities (GS autoencoder and TVP) effectively improve the forecasting performance. 
\subsection{Impulse Response Analysis}
\label{sec:Impulse Response Analysis}

Since the non-linear GS autoencoder is the best model in most forecasting tasks, we explore the impulse responses (IRF) inferred by this model using the whole data set. Firstly, we consider the responses of three variables included in the VAR model: GDPDEF, UNRATE and FEDFUNDS, to an expansionary monetary shock, which is a 100 bps decrease of the FEDFUNDS. Figure \ref{fig:irf_y} presents the evolution of IRFs over time. The medians of IRFs are in the first column, and we select three time points, 1981:Q3, 2000:Q4, 2020:Q1, for the rest columns. These three time points correspond to the representative rate cuts during the chairmanship of Volcker, Greenspan and Powell, respectively, and are separated by an approximately 20-year interval. We exclude the rate cuts during the chairmanship of Bernanke and Yellen in this analysis for two reasons. Firstly, the IRFs during the GFC, are similar to those in 2000:Q4\footnote{our result is consistent with that in \cite{korobilis2013assessing}, that the IRFs are similar during the chairmanship of Greenspan and Bernanke.}, albeit with greater uncertainty. Secondly, no significant rate cut occurred during the tenure of Yellen. Nevertheless, a comprehensive analysis of (IRFs) across various Federal Reserve chairmanships remains valuable for a holistic understanding, thus we include a figure with all regimes from Burns to Powell in Appendix \ref{sec:Additional Results of Real Data Application}.

For GDPDEF, the shape of the IRFs remains consistent over time. These IRFs are typically hump-shaped, starting at zero, reaching a peak at specific horizons, and decaying back to zero. However, we can still find the difference in the transmission of monetary policy from the next three columns. In the 1981 scenario, the IRF reached its peak 13 quarters after introducing the interest rate shock, while the IRFs for the other two periods peaked earlier at 10 quarters post-shock. The monetary policy had a larger impact on the GDPDEF in 1981 than other time periods, as evidenced by a higher peak response and more persistent effects of the shock over time. Although we also observe this resistance in 2000 and 2020, the effect is weaker with more uncertainty.
\begin{figure}[!htbp]
    \centering
    \includegraphics[width=\linewidth]{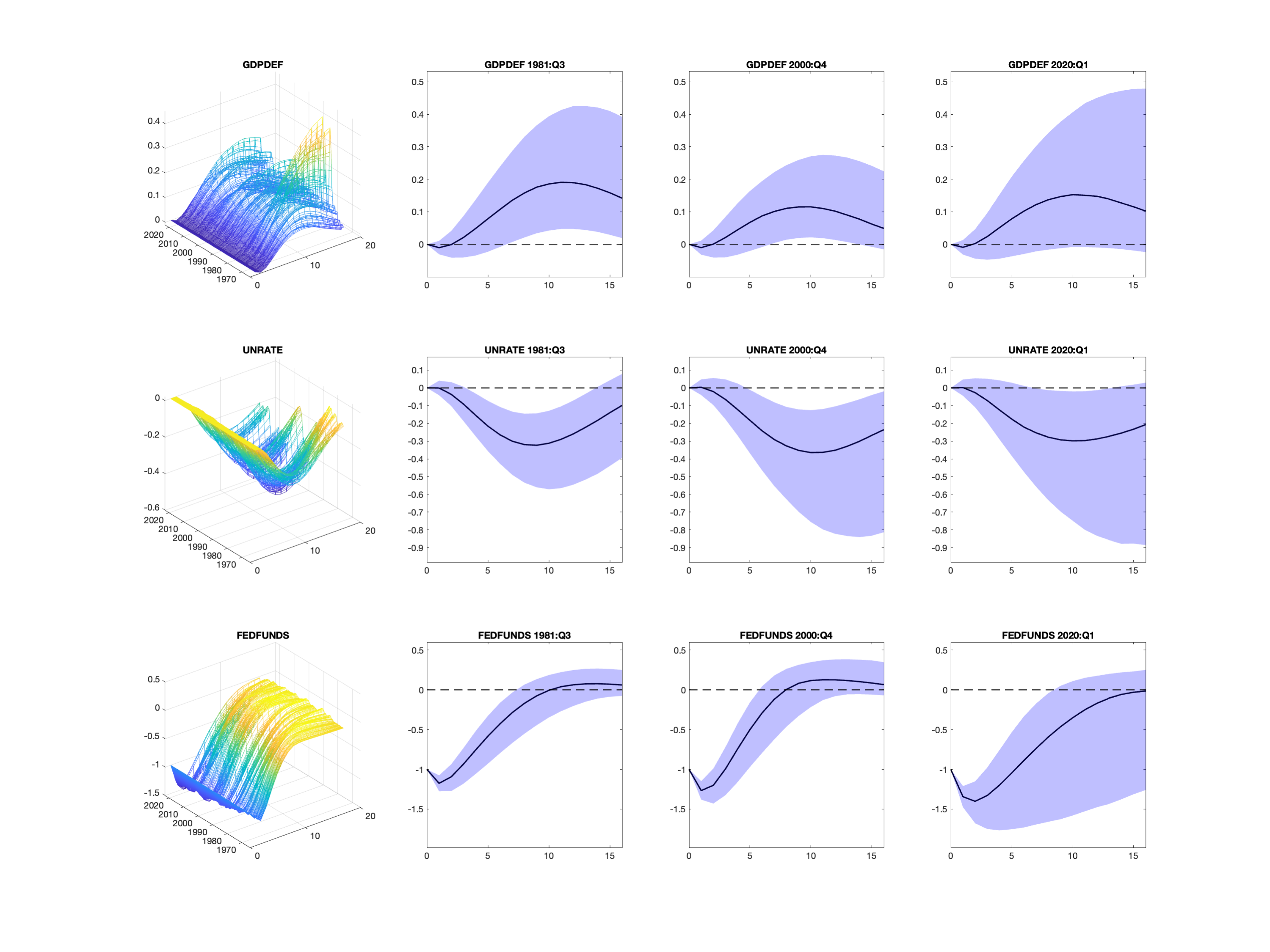}
    \caption{Impulse responses of the VAR variables to a 100 bps decrease in FEDFUNDS. First column shows the medians over time. The rest three columns shows the IRFs with their 68\% credible intervals at 1981:Q3, 2000:Q4 and 2020:Q1, respectively.}
    \label{fig:irf_y}
\end{figure}

The middle left panel in Figure \ref{fig:irf_y} reveals time variation in the IRFS of the UNRATE. During the recession periods, the UNRATE responded moderately, characterized by shallower troughs. Examining the three selected time points, we find that the IRFs in 1981 and 2000 gave similar patterns, with the latter showing more uncertainty. In contrast, the IRF for 2020 was not statistically significant from zero for most of the post-shock period.

Regarding the FEDFUNDS, the median IRFs gradually evolved in shape, progressing from a subtle to a more pronounced inverted hump-shaped pattern. The rate at which the responses decayed to zero is notably reduced. While the IRF in 1981 and 2000 crossed the zero line after approximately 8 quarters, the IRF for 2020 required 16 quarters to reach zero. An additional distinction between this IRF and the previous two is the higher uncertainty.

\begin{figure}[!htbp]
    \centering
    \includegraphics[width=\linewidth]{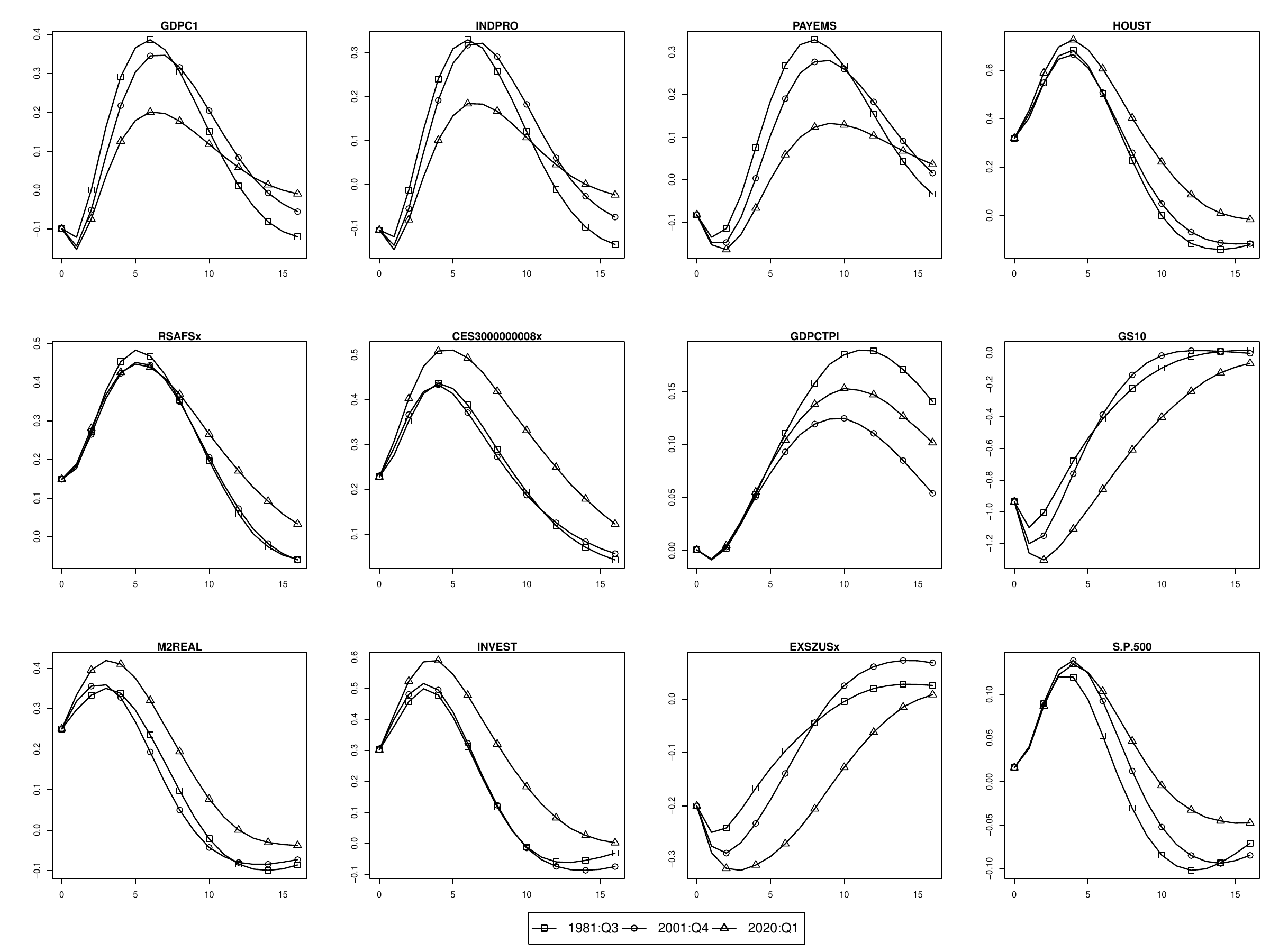}
    \caption{Impulse response medians of selected variables to a 100 bps decrease in FEDFUNDS at 1981:Q3, 2000:Q4 and 2020:Q1. }
    \label{fig:irf_x}
\end{figure}

Figure \ref{fig:irf_x} depicts the IRFs of a selection of variables at the three time points considered. Overall, the shapes and signs are consistent with results in \cite{christiano2005nominal} and \cite{bernanke2005measuring}. Unlike the finding in \cite{korobilis2013assessing} that some IRFs are time-varying and others are not, all IRFs show a certain degree of time variation in our case. A possible explanation is that we study a data set with longer time period and the choices of data points are very different, see \cite{korobilis2013assessing} for more details. Most IRFs are similar in 1980 and 2000, yet manifest a divergence in 2020. For example, the peaks of IRFs of the GDP, industrial production (INDPRO) and non-farm payroll employment (PAYEMS) were lower in 2020, compared to the ones for the other two time points, but the 2020 IRFs surpassed the others 16 quarters after the shock. The exception is the price index GDPCPTI, of which the IRFs had overlapping trajectories in the first 4 quarters, followed by divergent paths in the subsequent quarters.


\section{Conclusion and Discussion}
\label{sec:Conclusion and Discussion}

In this paper, we extend the FAVAR with an autoencoder by proposing a more interpretable variant called the Grouped Sparse autoencoder. Our model identifies factors up to element-wise transformation, and we exploit different transformation properties to select the activation function. To apply non-linearity to both the factor extraction and VAR parts of the FAVAR, we also adopt the TVP-VAR to model a time-varying evolution of factors. The empirical results suggest the model proposed has better interpretability and forecasting performance than the FAVARs with either a linear dimension reduction method or a TIV-VAR. This model also captures time variation in the variable responses to the monetary policy shocks. 

The current FAVAR framework can be extended in several directions. One direction is to relax the full column rank condition for the injective decoder by exploring alternative invertible or injective decoder architectures. Another enhancement would be replacing the TVP-VAR with neural networks to potentially achieve better expressiveness. Developing a one-step procedure for simultaneous factor extraction and parameter estimation across the entire FAVAR system presents another interesting direction. From an application perspective, exploring asymmetric impulse responses in the high-dimensional data merits investigation, as the current framework still assumes symmetric responses from the high-dimensional data to expansionary and recessionary shocks.





\appendix

\section{Further Description regarding Table \ref{tab:property}}\label{sec:property}

We elaborate the connection between the properties and their corresponding examples in this section. Starting from the case when the decoder is invertible, the closed form of the transformation is available in the proof of Theorem 1. The first example is straightforward, readers can refer to the literature about normalizing flow for more details. The second example is less obvious since normally a multilayer perceptron (MLP) without activation function is not invertible, but it is in the GS autoencoder. That is because we can rewrite this linear transformation to $\boldsymbol{x}_{t,i_k}=a_{i_k}\boldsymbol{f}_{t,k}+b_{i_k}$ for the $i_k$-th variable that is from the anchor group of the $k$-th factor.

For the injective decoders, we prove that those three conditions yield injective neural networks. \begin{proof}
To prove the decoder, $\tilde{g}^d_{i_k,\boldsymbol{\theta}}$, is injective, it is sufficient to prove that the function from one hidden layer to another is injective. Since the activation function is injective (the first condition), we can prove the above statement if and only if the corresponding functions without activation function, i.e. the linear transformations, are injective. 

Let these functions split to three parts: (1) from $\boldsymbol{f}_t$ to $\boldsymbol{h}^d_{t,1}$, (2) from $\boldsymbol{h}^d_{t,l}$ to $\boldsymbol{h}^d_{t,l+1}$, for $l=1,\dots, L-2$, and (3) from $\boldsymbol{h}^d_{t,L-1}$ to $\boldsymbol{h}^d_{t,L}=\boldsymbol{x}_{t,i_k}$. For the first part, we simplify it to: from $f_{t,k}$ to $\boldsymbol{h}^d_{t,1}$, since the $i_k$-th variable is from the anchor group, so the weight in this function changes from $\boldsymbol{W}^d_{1,i_k}\in \mathbb{R}^{D^d_1\times K}$ to its $k$-th column, i.e a $D^d_1$-by-1 matrix, which is a full column rank matrix. This property of the weights extend to the second part due to the second condition. Finally, if we modify the weight in the third part according to the third condition, then this weight has full column rank as well. \cite{strang2022introduction} showed that the linear transformation is injective if the corresponding matrix is a full column rank one. Thus, we can prove that the neural networks satisfying the three conditions are injective.
\end{proof}
Since the decoder is injective, we can always find a one-to-one mapping from $\boldsymbol{f}^*_t$ to $\boldsymbol{f}_t$ given the decoder parameters in the two distinct GS autoencoder. 

Next, we demonstrate that non-injective decoders, such as those with the ReLU, imply that the transformation, $h_{k}(\cdot)$ is one-to-many. For any $k\in\{1,\dots,K\}$ and $t\in\{1,\dots, T\}$, suppose $f_{t,k}=1$ or $2$ and $f^*_{t,k}=3$ result in the same reconstruction at time $t$. Assume that $\boldsymbol{\theta}$ and $\boldsymbol{\theta}^*$ are two sets of parameters of the decoders in these two GS autoencoder, respectively. For the first GS autoencoder, we may learn the parameters in the encoder as ${\boldsymbol{\phi}}$ that produces $f_{t,k}=1$ or $\hat{\boldsymbol{\phi}}$ that produces $f_{t,k}=2$. Without considering the properties and parameters in the encoder, we can only get a one-to-many transformation, $h(\cdot)$, that $h(3)=1$ or 2.

\section{Estimation Details}
\label{sec:Estimation Details}
\subsection{Derivation of the ELBO}
\label{sec:Derivation of the ELBO}
\begin{align}
    \log p\left(\boldsymbol{x}_{1:T}\right) &= \log \int\int p\left(\boldsymbol{x}_{1:T}\mid \boldsymbol{B}\right)p\left(\boldsymbol{B}\mid \boldsymbol{\Gamma}\right)p\left(\boldsymbol{\Gamma}\right)\,d\boldsymbol{\Gamma}d\boldsymbol{B} \nonumber \\
    &=\log \int p\left(\boldsymbol{x}_{1:T}\mid \boldsymbol{B}\right) \int p\left(\boldsymbol{B}\mid \boldsymbol{\Gamma}\right)p\left(\boldsymbol{\Gamma}\right)\,d\boldsymbol{\Gamma}d\boldsymbol{B} \nonumber \\
    &=\log \int p\left(\boldsymbol{x}_{1:T}\mid \boldsymbol{B}\right) \mathbb{E}_{p\left(\boldsymbol{\Gamma}\mid\boldsymbol{B}\right)}\left[\frac{p\left(\boldsymbol{B}\mid \boldsymbol{\Gamma}\right)p\left(\boldsymbol{\Gamma}\right)}{p\left(\boldsymbol{\Gamma}\mid\boldsymbol{B}\right)}\right]d\boldsymbol{B} \nonumber \\
    &=\log \mathbb{E}_{q\left(\boldsymbol{B}\mid \boldsymbol{x}_{1:T}\right)}\left[p\left(\boldsymbol{x}_{1:T}\mid \boldsymbol{B}\right) \mathbb{E}_{p\left(\boldsymbol{\Gamma}\mid\boldsymbol{B}\right)}\left[\frac{p\left(\boldsymbol{B}\mid \boldsymbol{\Gamma}\right)p\left(\boldsymbol{\Gamma}\right)}{p\left(\boldsymbol{\Gamma}\mid\boldsymbol{B}\right)}\right]\right] \nonumber \\
    &\geq \mathbb{E}_{q\left(\boldsymbol{B}\mid \boldsymbol{x}_{1:T}\right)}\left[\log p\left(\boldsymbol{x}_{1:T}\mid \boldsymbol{B}\right)+\log \mathbb{E}_{p\left(\boldsymbol{\Gamma}\mid\boldsymbol{B}\right)}\left[\frac{p\left(\boldsymbol{B}\mid \boldsymbol{\Gamma}\right)p\left(\boldsymbol{\Gamma}\right)}{p\left(\boldsymbol{\Gamma}\mid\boldsymbol{B}\right)}\right]\right] \nonumber \\
    &\geq \mathbb{E}_{q\left(\boldsymbol{B}\mid \boldsymbol{x}_{1:T}\right)}\left[\log p\left(\boldsymbol{x}_{1:T}\mid \boldsymbol{B}\right) + \mathbb{E}_{p\left(\boldsymbol{\Gamma}\mid\boldsymbol{B}\right)}\left[\log \frac{p\left(\boldsymbol{B}\mid \boldsymbol{\Gamma}\right)p\left(\boldsymbol{\Gamma}\right)}{p\left(\boldsymbol{\Gamma}\mid\boldsymbol{B}\right)}\right]\right] \nonumber \\ 
    & = \underbrace{\mathbb{E}_{q\left(\boldsymbol{B}\mid \boldsymbol{x}_{1:T}\right)}\left[\log p\left(\boldsymbol{x}_{1:T}\mid \boldsymbol{B}\right)\right]}_{\text{Part 1}}  + \underbrace{\mathbb{E}_{q\left(\boldsymbol{B}\mid \boldsymbol{x}_{1:T}\right)}\left[ \mathbb{E}_{p\left(\boldsymbol{\Gamma}\mid\boldsymbol{B}\right)}\left[\log \frac{p\left(\boldsymbol{B}\mid \boldsymbol{\Gamma}\right)p\left(\boldsymbol{\Gamma}\right)}{p\left(\boldsymbol{\Gamma}\mid\boldsymbol{B}\right)}\right]\right]}_{\text{Part 2}}, \label{equ: ELBO_derived}
\end{align}
where the inequality is because logarithm is concave. In the $m$-th update of parameters,  $q\left(\boldsymbol{B}\mid \boldsymbol{x}_{1:T}\right) = 1$ when $\boldsymbol{B}=\boldsymbol{B}^{(m)}$ and $0$ otherwise. 

To simplify the notation, we discard the superscript $(m)$ when we expand \labelcref{equ: ELBO_derived}. Since we assume $\boldsymbol{x}_{1:T}$ is i.i.d across time and variables, the first part is:
\begin{align*}
    \mathbb{E}_{q\left(\boldsymbol{B}\mid \boldsymbol{x}_{1:T}\right)}\left[\log p\left(\boldsymbol{x}_{1:T}\mid \boldsymbol{B}\right)\right]=-\frac{1}{2}\sum_{t=1}^T\sum_{i=1}^N \left[\left(\boldsymbol{x}_{t,i}-\hat{\boldsymbol{x}}_{t,i}\right)^2-\log 2\pi\right].
\end{align*}
Denote $p_{c,k}=\mathbb{E}[\gamma_{c,k}\mid \beta_{c,k}]=\frac{\psi_1\left(\beta_{c,k}\right)}{\psi_0\left(\beta_{c,k}\right)+\psi_1\left(\beta_{c,k}\right)}$, then we can write the second part as:
\begin{align*}
&\sum_{c=1}^C\sum_{k=1}^K\mathbb{E}_{p\left(\gamma_{c,k}\mid\beta_{c,k}\right)}\left[\log \frac{p\left(\beta_{c,k}\mid \gamma_{c,k}\right)p\left(\gamma_{c,k}\right)}{p\left(\gamma_{c,k}\mid\beta_{c,k}\right)}\right]\\
&=\sum_{c=1}^C\sum_{k=1}^K p_{c,k}\left(\log \psi_1\left(\beta_{c,k}\right)-\log p_{c,k}\right)+\left(1-p_{c,k}\right) \left(\log \psi_0\left(\beta_{c,k}\right)-\log (1-p_{c,k})\right)-\log2.
\end{align*}
Take out the constant terms and scale two parts by $T$ and $N$, we can get the objective function in \labelcref{equ:ELBO}.


\subsection{Bayesian Inference of Factor Model}

Given the factor model in \labelcref{equ:FAVAR} and the priors in Section \ref{sec:Bayesian Inference}, the full conditional of $\boldsymbol{\Lambda}_i$, the $i$-th row of $\Lambda$ for $i=1,\dots, N$, is $\mathcal{N}\left(\overline{m}_i, \overline{V}_i\right)$, with
\begin{align}
\overline{V}^{-1}_i&=\frac{1}{4}\boldsymbol{I}+\boldsymbol{\Sigma}^{-1}_{i,i}\left(\boldsymbol{F}, \boldsymbol{Y}\right)^\prime\left(\boldsymbol{F}, \boldsymbol{Y}\right),\\
\overline{m}_i& = \boldsymbol{\Sigma}^{-1}_{i,i}\overline{V}_i\left(\boldsymbol{F}, \boldsymbol{Y}\right)^\prime\boldsymbol{X},
\end{align}
where $\boldsymbol{Y}=\left(\boldsymbol{y}_1,\dots, \boldsymbol{y}_T\right)^\prime$ and $\boldsymbol{X}=\left(\boldsymbol{x}_1,\dots, \boldsymbol{x}_T\right)^\prime$.


\subsection{Bayesian Inference of Time-invariant VAR}
\label{sec:Bayesian Inference of Time-invariant VAR}
The coefficient $\boldsymbol{A}_p$, $p=1,\dots,P$, follows a Minnesota-type prior, $\boldsymbol{A}_{p,(i,j)}\sim\mathcal{N}\left(0, V_{p,(i,j)}\right)$, where $V_{p,(i,j)}=
\begin{cases}
    \frac{\xi_1}{p^2}, & \text{if }i=j\\
    \frac{\xi_2}{p^2}\frac{\hat{\sigma}_i}{\hat{\sigma}_j}, &\text{if }i\neq j
    \end{cases}$, where $\hat{\sigma}^2_i$ is the variance estimate of $\boldsymbol{y}_{t,i}$ sequence modelled by an AR(2) process, $\xi_1$ and $\xi_2$ are hyperparameters with prior Gamma$\left(0.01,0.01\right)$. The variance-covariance matrix $\boldsymbol{\Omega}$ follows an non-informative Inverse-Wishart prior, i.e. $\mathcal{IW}\left(0\boldsymbol{I}, 0\right)$.

The vectorization of $\boldsymbol{A}^\prime$, $\boldsymbol{a}$, has the full conditional $\mathcal{N}\left(\overline{m}, \overline{\boldsymbol{V}}\right)$, with
\begin{align}
    \overline{\boldsymbol{V}}^{-1}&=\underline{\boldsymbol{V}}^{-1}+\boldsymbol{Z}^\prime\left(\boldsymbol{I}\otimes \boldsymbol{\Omega}^{-1}\right)\boldsymbol{Z},\\
    \overline{\boldsymbol{m}}&=\overline{\boldsymbol{V}}\left(\boldsymbol{Z}^{-1}\left(\boldsymbol{I}\otimes \boldsymbol{\Omega}^{-1}\right)\text{vec}\left(\boldsymbol{F}\,\boldsymbol{Y}\right)^\prime\right),
\end{align}
where $\boldsymbol{Z}=\left(\boldsymbol{Z}^\prime_1,\dots, \boldsymbol{Z}^\prime_T\right)^\prime$, $\boldsymbol{Z}_t=\boldsymbol{I}_{M+K}\otimes \left(\boldsymbol{f}^\prime_{t-1}, \boldsymbol{y}^\prime_{t-1},\dots, \boldsymbol{f}^\prime_{t-P}, \boldsymbol{y}^\prime_{t-P}\right)$.

We sample $\boldsymbol{\Omega}^{-1}$ from $\mathcal{W}\left(T, \sum_{t=0}^T\left(\boldsymbol{y}_t-\boldsymbol{Z}_t\boldsymbol{a}\right)\left(\boldsymbol{y}_t-\boldsymbol{Z}_t\boldsymbol{a}\right)^\prime\right)$. The inference of $\xi_1$ and $\xi_2$ follows  \cite{giannone2015prior}.

\section{Cross Validation}
\label{sec:Cross Validation}
We split the cross-validation to two parts to determine: 1) the numbers of factors and the number of hidden layers ($L$), activation function, and 2) $\lambda_0$ and $\lambda_1$, respectively. We use 5-fold cross-validation to standard autoencoders to determine the first set of hyperparameters, then apply the same technique to the grouped sparse autoencoder with the first set of hyperparameters determined to find the second set of hyperparameters. This cross-validation allows us to make the standard and grouped sparse autoencoder comparable and save computational time. 

We choose the number of factors from 2 to 5 and $L$ from 2 to 6. The activation function are selected from tanh and LeakyReLU with the multiplier as 0.01 or  ${10}^{-16}$ (to mimic the ReLU but retain the preferable properties in Table \ref{tab:property}). Table \ref{tab:cross validation} records the averaged loss from each model settings, and suggests that the best-performed model is the one with 5 factors, 3 layers and LeakyReLU ($a={10}^{-16}$). Then, we use this architecture to determine the values of $\lambda_0$ and $\lambda_1$, from $\{100, 500, 1000\}$ and $\{1,0.1,0.01,0.001\}$, respectively. Table \ref{tab:cross validation 1} shows that $\lambda_0=1000$ and $\lambda_1=1$ achieves the lowest validation loss. 

\begin{table}[!htbp]
\footnotesize
\centering
\begin{tabular}{ccccc}
\toprule
Tanh & & & & \\
\hline
& K=2                     & K=3                     & K=4                     & K=5                     \\
L=2                                              & {\color[HTML]{212121} 0.531} & {\color[HTML]{212121} 0.432} & {\color[HTML]{212121} 0.387} & {\color[HTML]{212121} 0.309} \\
L=3                                              & {\color[HTML]{212121} 0.510} & {\color[HTML]{212121} 0.429} & {\color[HTML]{212121} 0.364} & {\color[HTML]{212121} 0.300} \\
L=4                                              & {\color[HTML]{212121} 0.536} & {\color[HTML]{212121} 0.403} & {\color[HTML]{212121} 0.331} & {\color[HTML]{212121} 0.283} \\
L=5                                              & {\color[HTML]{212121} 0.501} & {\color[HTML]{212121} 0.380} & {\color[HTML]{212121} 0.341} & {\color[HTML]{212121} 0.278} \\
L=6                                              & {\color[HTML]{212121} 0.494} & {\color[HTML]{212121} 0.396} & {\color[HTML]{212121} 0.394} & {\color[HTML]{212121} 0.353} \\
\midrule
\multicolumn{3}{l}{\cellcolor[HTML]{FFFFFF}Leaky ReLU ($a=0.01$)} & \multicolumn{1}{l}{}         & \multicolumn{1}{l}{}           \\ 
\hline
\rule{0pt}{1\normalbaselineskip}
 & K=2                     & K=3                     & K=4                     & K=5                     \\
L=2                                              & {\color[HTML]{212121} 0.480} & {\color[HTML]{212121} 0.399} & {\color[HTML]{212121} 0.318} & {\color[HTML]{212121} 0.267} \\
L=3                                              & {\color[HTML]{212121} 0.465} & {\color[HTML]{212121} 0.364} & {\color[HTML]{212121} 0.284} & {\color[HTML]{212121} 0.246} \\
L=4                                              & {\color[HTML]{212121} 0.441} & {\color[HTML]{212121} 0.392} & {\color[HTML]{212121} 0.290} & {\color[HTML]{212121} 0.254} \\
L=5                                              & {\color[HTML]{212121} 0.429} & {\color[HTML]{212121} 0.329} & {\color[HTML]{212121} 0.290} & {\color[HTML]{212121} 0.282} \\
L=6                                              & {\color[HTML]{212121} 0.464} & {\color[HTML]{212121} 0.360} & {\color[HTML]{212121} 0.336} & {\color[HTML]{212121} 0.302} \\
\midrule
\multicolumn{3}{l}{\cellcolor[HTML]{FFFFFF}Leaky ReLU ($a=10^{-16}$)} & \multicolumn{1}{l}{}         & \multicolumn{1}{l}{}           \\ 
\hline
\rule{0pt}{1\normalbaselineskip}
  & K=2                     & K=3                     & K=4                     & K=5                     \\
L=2                                              & {\color[HTML]{212121} 0.507} & {\color[HTML]{212121} 0.405} & {\color[HTML]{212121} 0.332} & {\color[HTML]{212121} 0.271} \\
L=3                                              & {\color[HTML]{212121} 0.477} & {\color[HTML]{212121} 0.361} & {\color[HTML]{212121} 0.286} & {\color[HTML]{212121} 0.244} \\
L=4                                              & {\color[HTML]{212121} 0.477} & {\color[HTML]{212121} 0.398} & {\color[HTML]{212121} 0.271} & {\color[HTML]{212121} 0.253} \\
L=5                                              & {\color[HTML]{212121} 0.444} & {\color[HTML]{212121} 0.360} & {\color[HTML]{212121} 0.280} & {\color[HTML]{212121} 0.265} \\
L=6                                              & {\color[HTML]{212121} 0.455} & {\color[HTML]{212121} 0.342} & {\color[HTML]{212121} 0.321} & {\color[HTML]{212121} 0.324}
\\
\bottomrule
\end{tabular}
\caption{Cross validation results from different combinations of activation function and hyperparameters.}
\label{tab:cross validation}
\end{table}

\begin{table}[!htbp]
\footnotesize
\centering
\begin{tabular}{lccc}
\toprule
            & \multicolumn{1}{l}{$\lambda_0$ = 100} & \multicolumn{1}{l}{$\lambda_0$ = 500} & \multicolumn{1}{l}{$\lambda_0$ = 1000} \\
            \hline
            \rule{0pt}{0.3\normalbaselineskip}\\
$\lambda_1$ = 1     & 0.195                             & 0.200                             & 0.187                              \\
$\lambda_1$ = 0.1   & 0.227                             & 0.217                             & 0.209                              \\
$\lambda_1$ = 0.01  & 0.237                             & 0.238                             & 0.230                              \\
$\lambda_1$ = 0.001 & 0.268                             & 0.250                             & 0.246  \\
\bottomrule
\end{tabular}
\caption{Cross validation results from different combination of $\lambda_0$ and $\lambda_1$.}
\label{tab:cross validation 1}
\end{table}

\section{Additional Results of Real Data Application}
\label{sec:Additional Results of Real Data Application}

Figure \ref{fig: factor importance linear grouped sparse ae} shows the importance measures of factors extracted from the linear grouped sparse autoencoder to different categories. The importance measures have the similar extent of sparsity like the one from the non-linear model. The factors can also be named according to the anchor groups. However, the way that factors reconstruct the high-dimensional data is different. For example, the labor market factor is relatively less important than other factors due to low measures in Figure \ref{fig: factor importance linear grouped sparse ae}, but the least important factor changes to the money and credit factor if we consider the non-linear method. Move to different groups, the linear model suggests that main driver of housing is the prices factor, while NIPA and labour market factors are the ones inferred from the non-linear model. Similar difference can also be found in other categories. 

\begin{figure}[!htbp]
    \centering
    \includegraphics[width=0.48\linewidth]{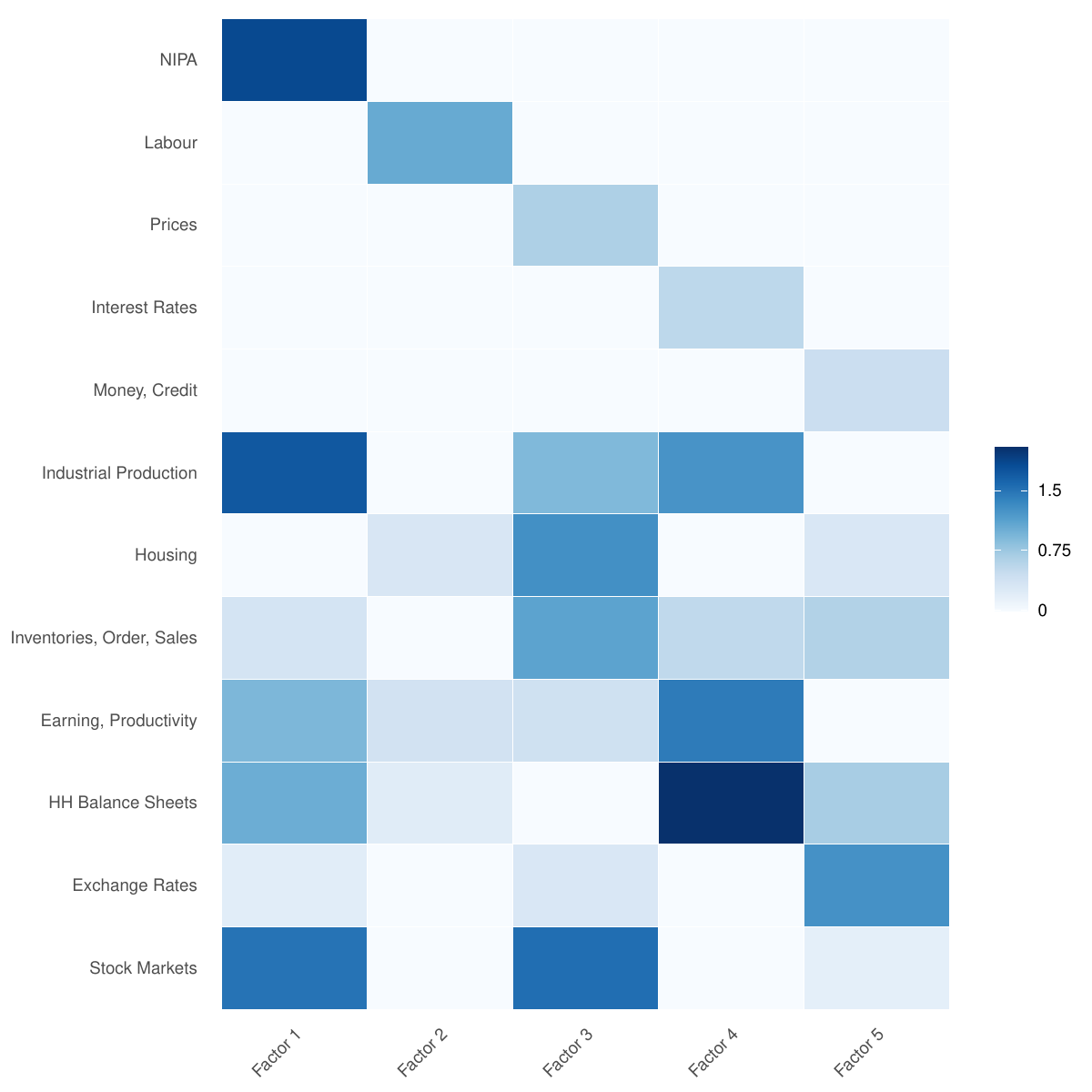}
    \caption{Importance of factors to different categories. "HH Balance Sheets" means houseshold balance sheets. The importance measures corresponds to elements of $\boldsymbol{B}$.}
    \label{fig: factor importance linear grouped sparse ae}
\end{figure}

The first 4 factors from linear and non-linear grouped sparse autoencoder are highly correlated, so the corresponding variables are almost the same up to permutation, but there are a few differences. For the labor market factor, the non-linear factor is a smoother version of the linear one with more weight to the COVID-19 pandemic. The troughs around 1986 and the GFC are lower in the non-linear price factors than the linear one. Unlike the labor market factors, the non-linear interest rates factor fluctuates more than its linear counterpart, with more pronounced downturns during the three most recent recession periods. Similar to the difference in Figure \ref{fig:factor_ts_nonlinesr_5}, the trends of money and credit factors deviate after 1995 from positive to negative correlation. 

\begin{figure}[!htbp]
\centering
\includegraphics[width=\textwidth]{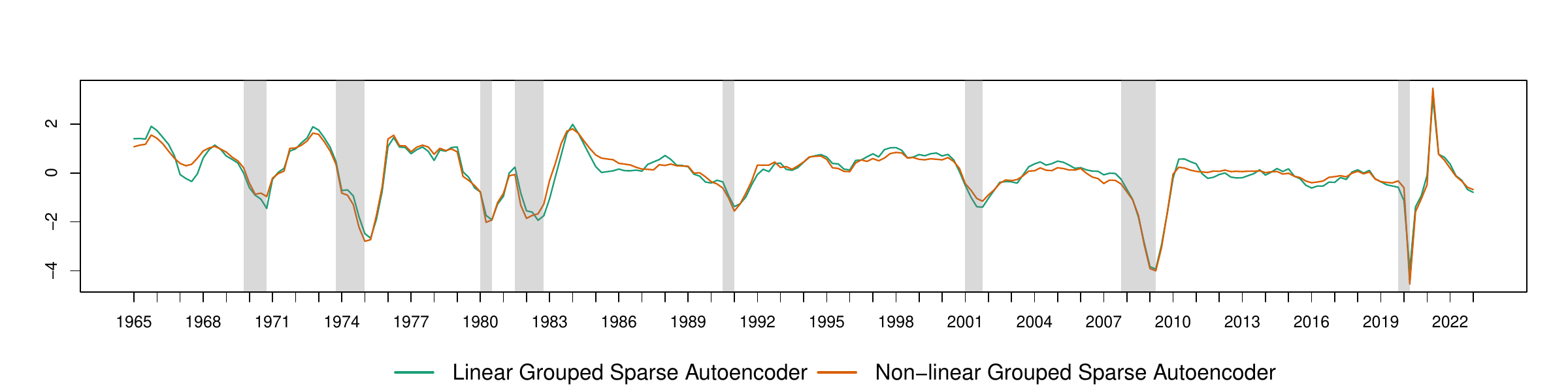}
\vskip 0.5cm
\includegraphics[width=0.8\textwidth]{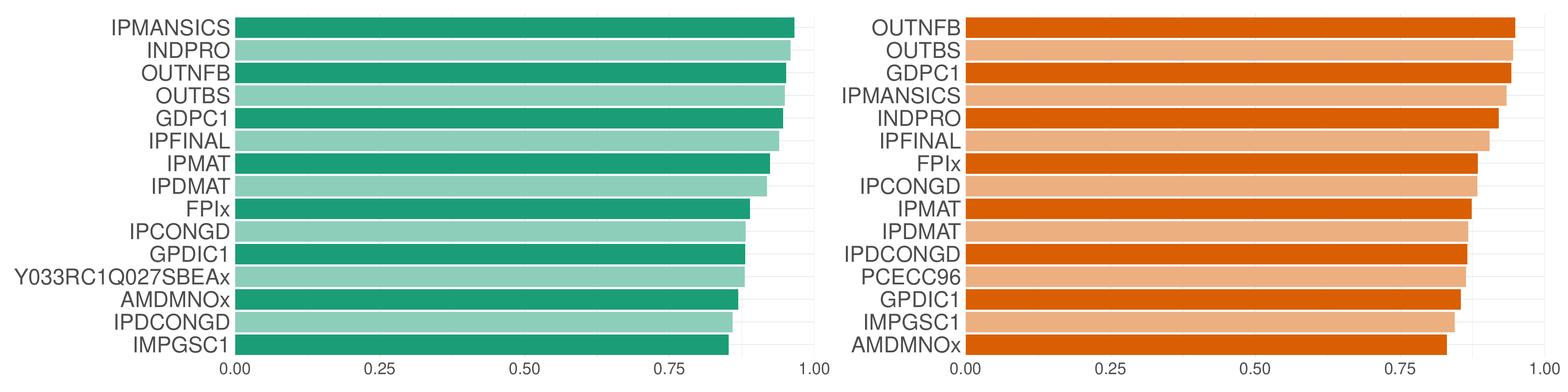}
\vskip\floatsep
\caption{The first factor extracted from the linear and non-linear GS autoencoder (top panel), and variables with the 15 highest correlation magnitudes with the corresponding factors (bottom panel). The time series are standardized to have zero mean and variance one. The grey bands highlight the recession periods.}
\label{fig:factor_ts_nonlinesr_1_appendix}
\end{figure}

\begin{figure}[!htbp]
\centering
\includegraphics[width=\textwidth]{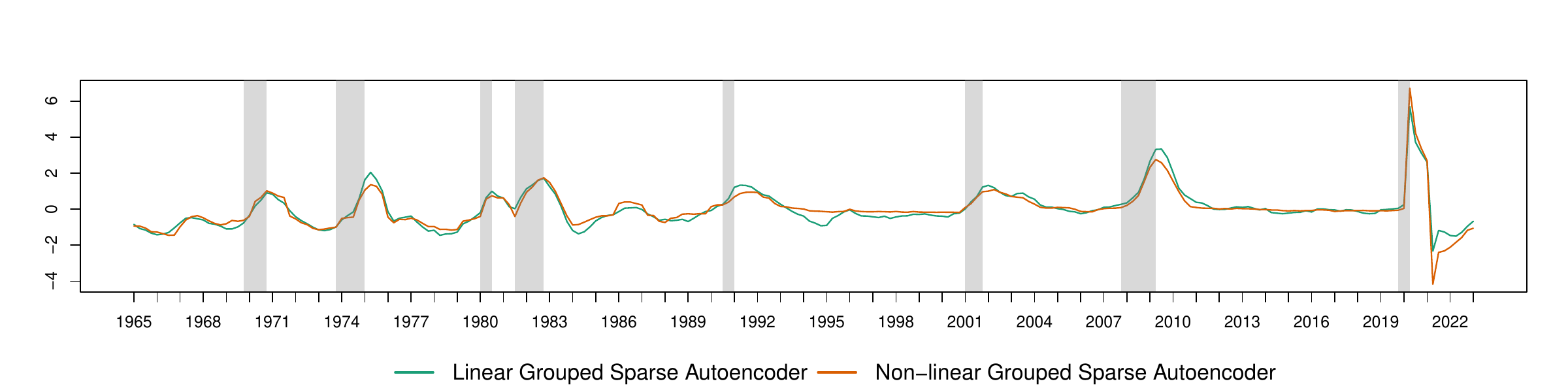}
\vskip 0.5cm
\includegraphics[width=0.8\textwidth]{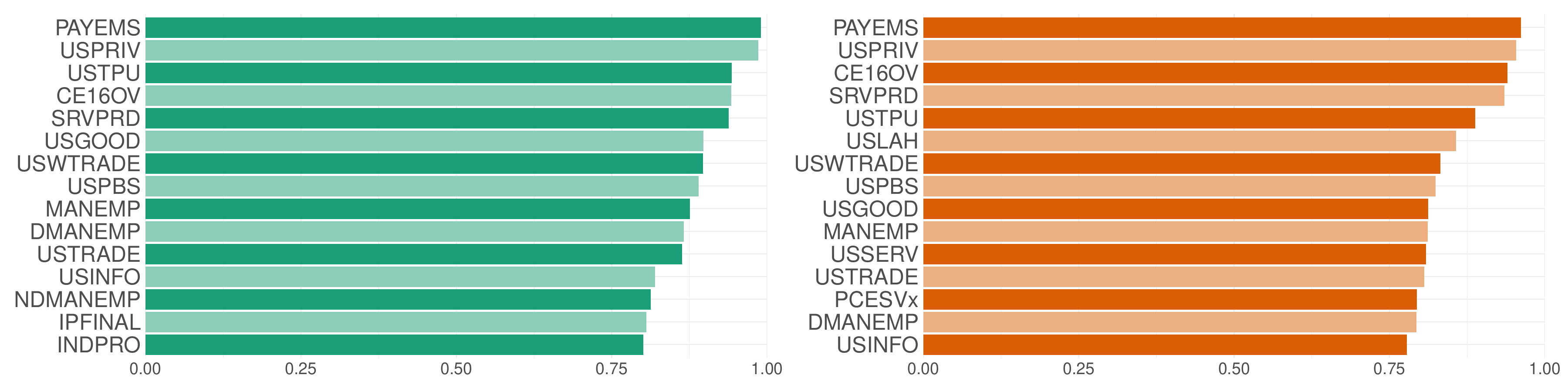}
\vskip\floatsep
\caption{The second factor extracted from the linear and non-linear GS autoencoder (top panel), and variables with the 15 highest correlation magnitudes with the corresponding factors (bottom panel). The time series are standardized to have zero mean and variance one. The grey bands highlight the recession periods.}
\label{fig:factor_ts_nonlinesr_2_appendix}
\end{figure}

\begin{figure}[!htbp]
\centering
\includegraphics[width=\textwidth]{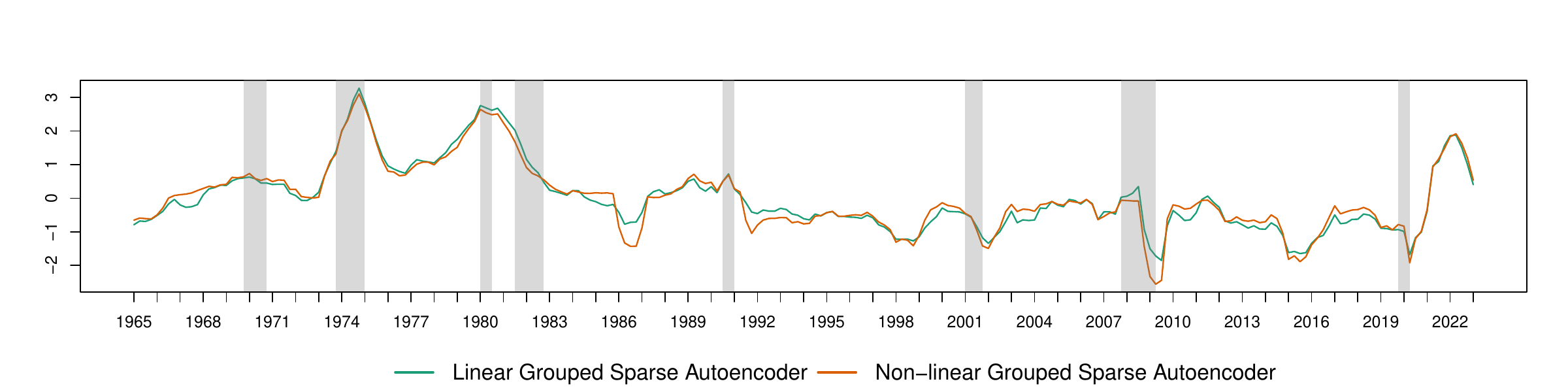}
\vskip 0.5cm
\includegraphics[width=0.8\textwidth]{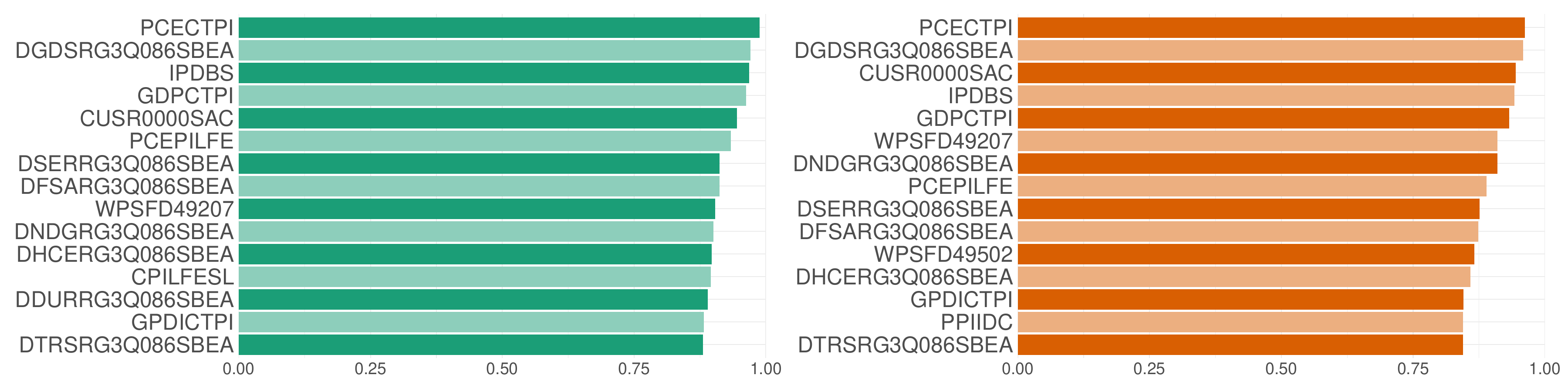}
\vskip\floatsep
\caption{The third factor extracted from the linear and non-linear GS autoencoder (top panel), and variables with the 15 highest correlation magnitudes with the corresponding factors (bottom panel). The time series are standardized to have zero mean and variance one. The grey bands highlight the recession periods.}
\label{fig:factor_ts_nonlinesr_3_appendix}
\end{figure}

\begin{figure}[!htbp]
\centering
\includegraphics[width=\textwidth]{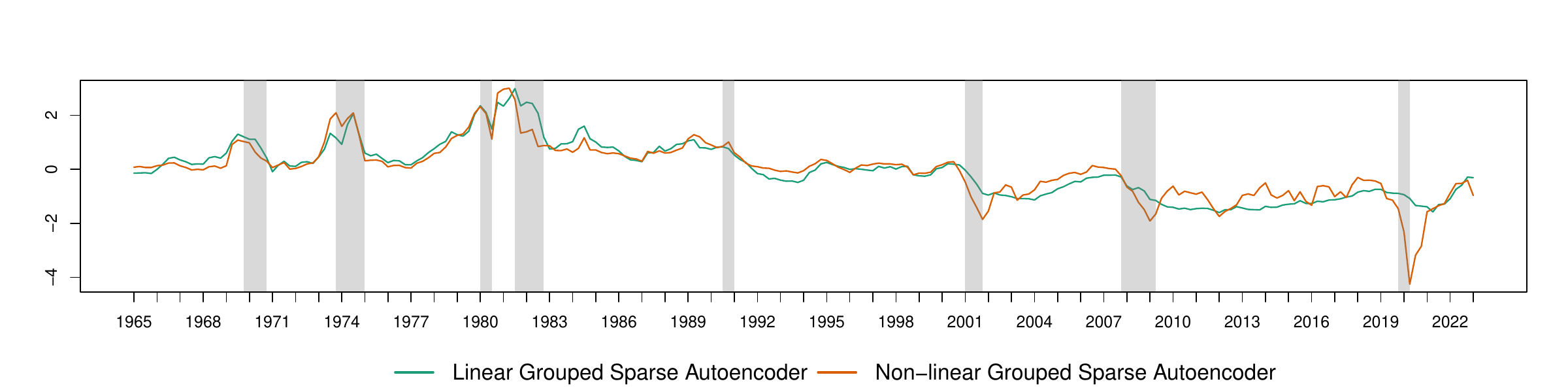}
\vskip 0.5cm
\includegraphics[width=0.8\textwidth]{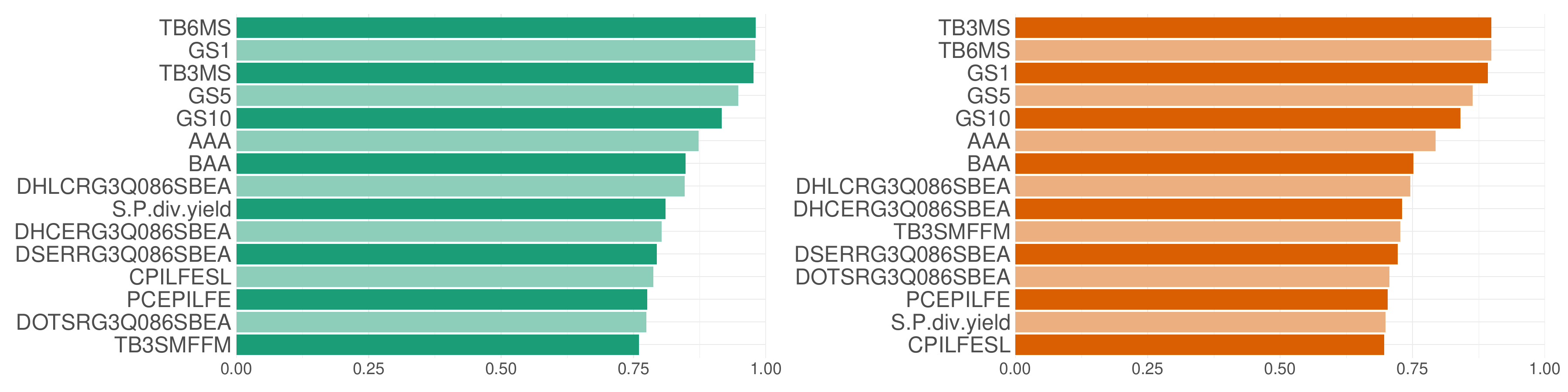}
\vskip\floatsep
\caption{The fourth factor extracted from the linear and non-linear GS autoencoder (top panel), and variables with the 15 highest correlation magnitudes with the corresponding factors (bottom panel). The time series are standardized to have zero mean and variance one. The grey bands highlight the recession periods.}
\label{fig:factor_ts_nonlinesr_4_appendix}
\end{figure}

\begin{figure}[!htbp]
\centering
\includegraphics[width=\textwidth]{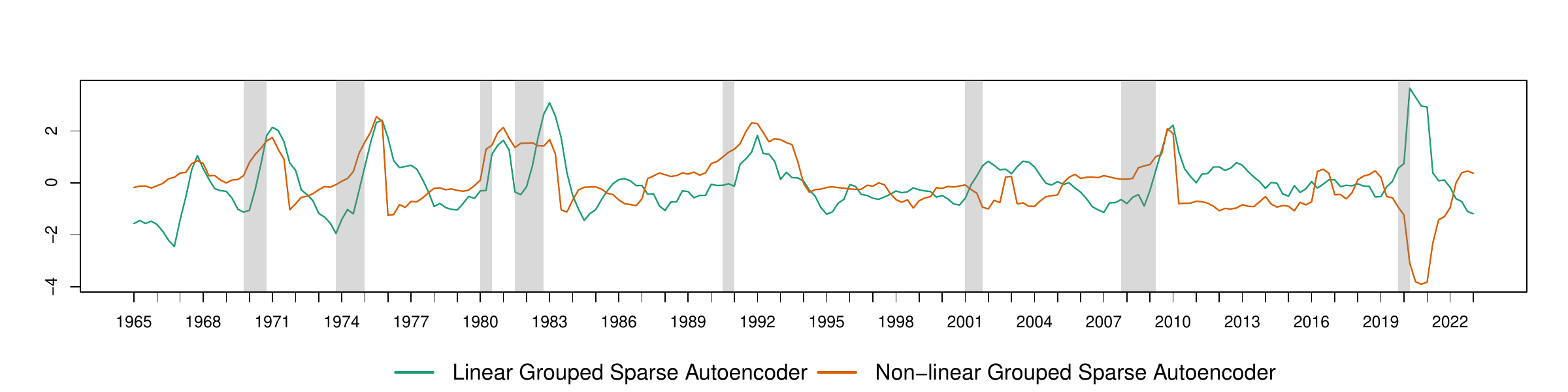}
\vskip 0.5cm
\includegraphics[width=0.8\textwidth]{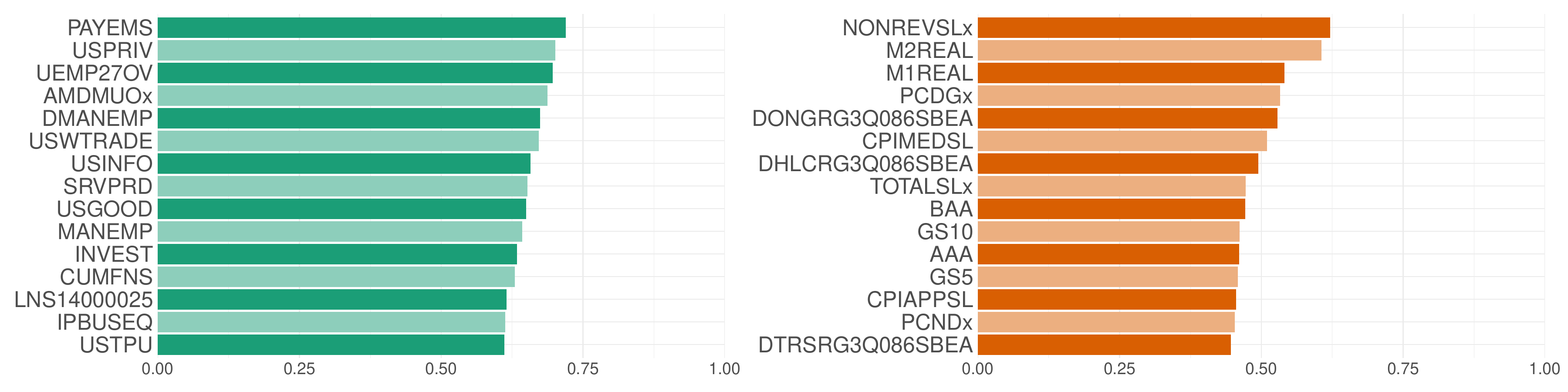}
\vskip\floatsep
\caption{The fifth factor extracted from the linear and non-linear GS autoencoder (top panel), and variables with the 15 highest correlation magnitudes with the corresponding factors (bottom panel). The time series are standardized to have zero mean and variance one. The grey bands highlight the recession periods.}
\label{fig:factor_ts_nonlinesr_5_appendix}
\end{figure}

The difference between the linear and non-linear models can also be found in the forecasting performance as in Section \ref{sec:Forecasting Performance} and the loss curves as shown in Figure \ref{fig: loss_curves}. This figure shows the total loss on the left panel and its decomposition: reconstruction loss on the middle and regularization loss on the right panel. To ensure that the model performance is not influenced by a particular set of parameter initializations, we train the models 100 times with different initializations and average the losses before generating the plots. The total loss is lower in the non-linear model than the linear one, and this discrepancy is from the reconstruction loss, since the regularization counterparts are similar. Thus, there exists non-linearity in the economic model and can be captured by the non-linear model.
\begin{figure}[!htbp]
    \centering
    \includegraphics[width=\linewidth]{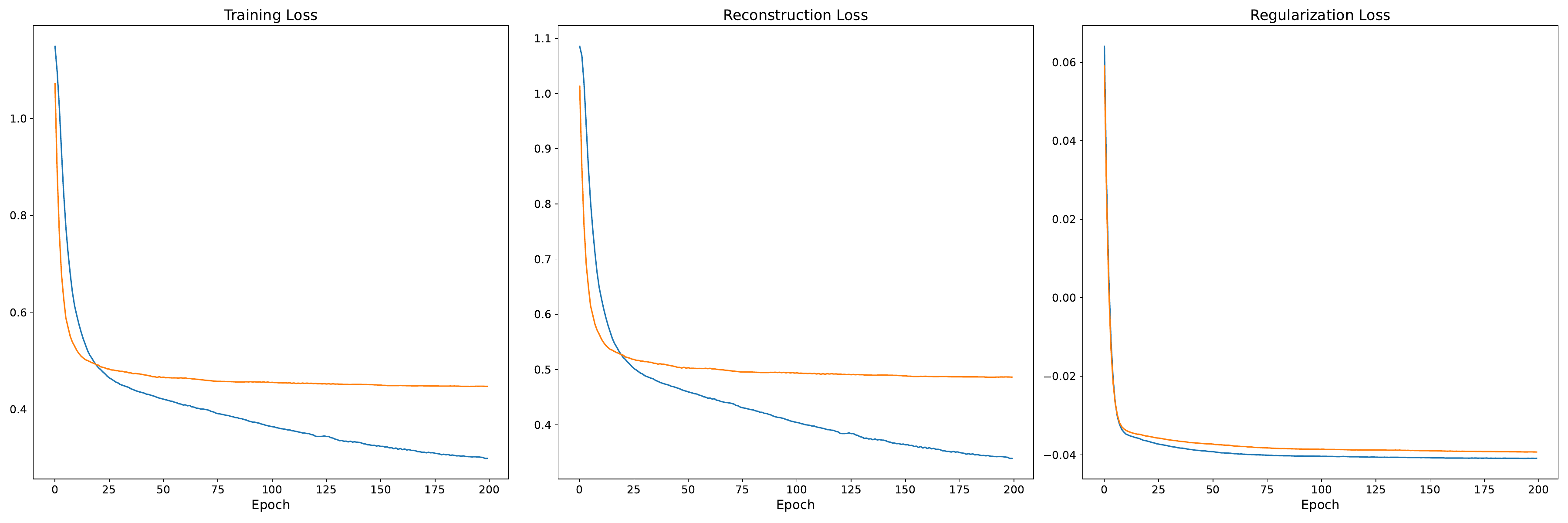}
    \caption{Loss curves against epochs from the non-linear (blue) and linear (orange) grouped sparse autoencoder}
    \label{fig: loss_curves}
\end{figure}

Figure \ref{fig:cumulative_ALPL_h=1} and Figure \ref{fig:cumulative_ALPL_h=4} report the quarterly and yearly density forecasting performance with similar findings from the quarterly one in Figure \ref{fig:cumulative_ALPL_h=2}. The overall findings in Section \ref{sec:Forecasting Performance} still hold, but we observe two differences. Firstly, as the forecasting horizon increases, the linear grouped sparse autoencoder becomes the best model in forecasting real activities without heteroskedasticity, while the PCA got worse in forecasting FEDFUNDS. Secondly, the GS autoencoder with the TVP got better performance in forecasting UNRATE before the COVID-19 pandemic, but their performance deteriorate afterward. 

\begin{figure}[!htbp]
    \centering
    \includegraphics[width=\linewidth]{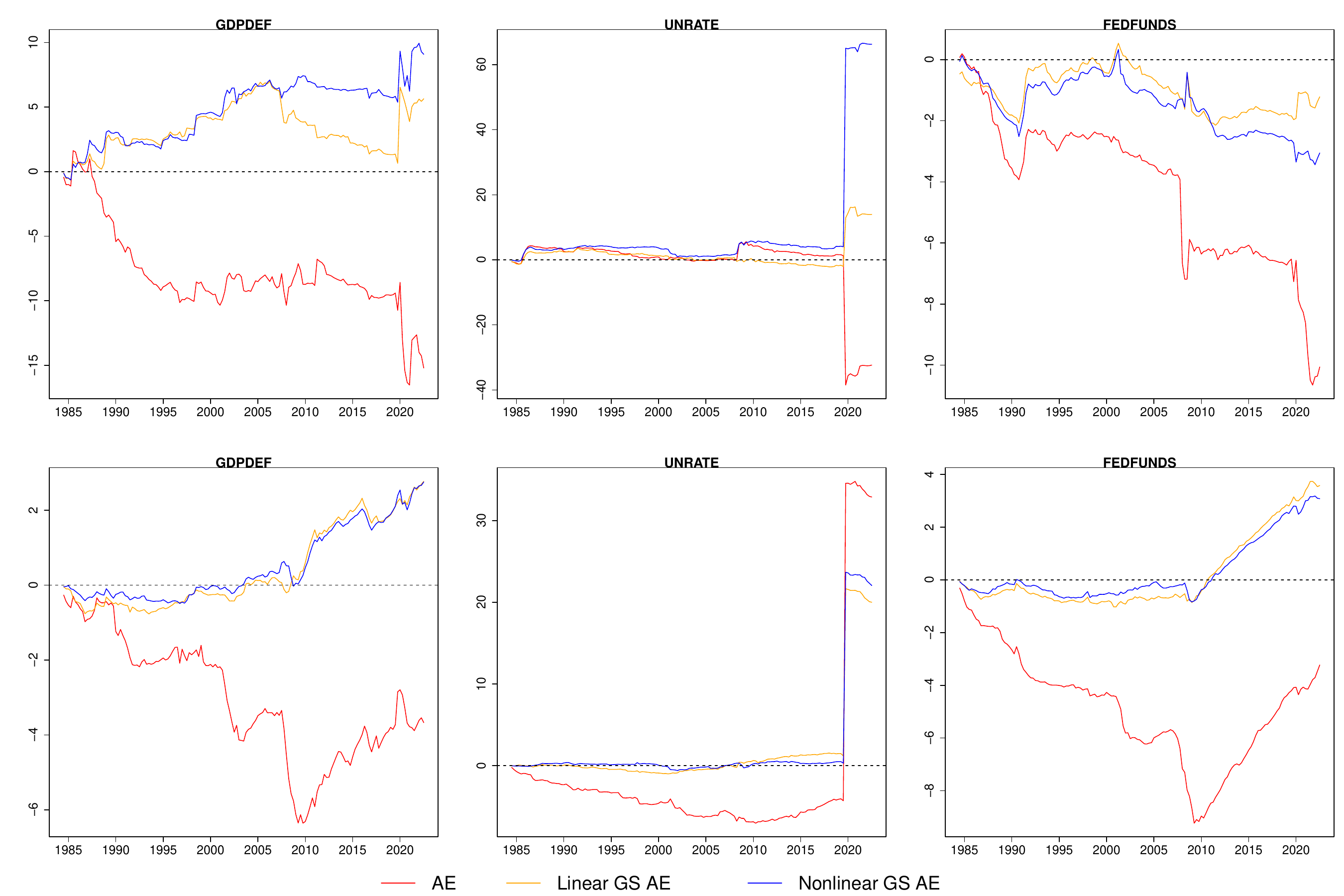}
    \caption{Cumulative ALPL (h=1) of models relative to their PCA counterparts. The top panels consider TIV models and the bottom ones consider TVP models.}
    \label{fig:cumulative_ALPL_h=1}
\end{figure}

\begin{figure}[!htbp]
    \centering
    \includegraphics[width=\linewidth]{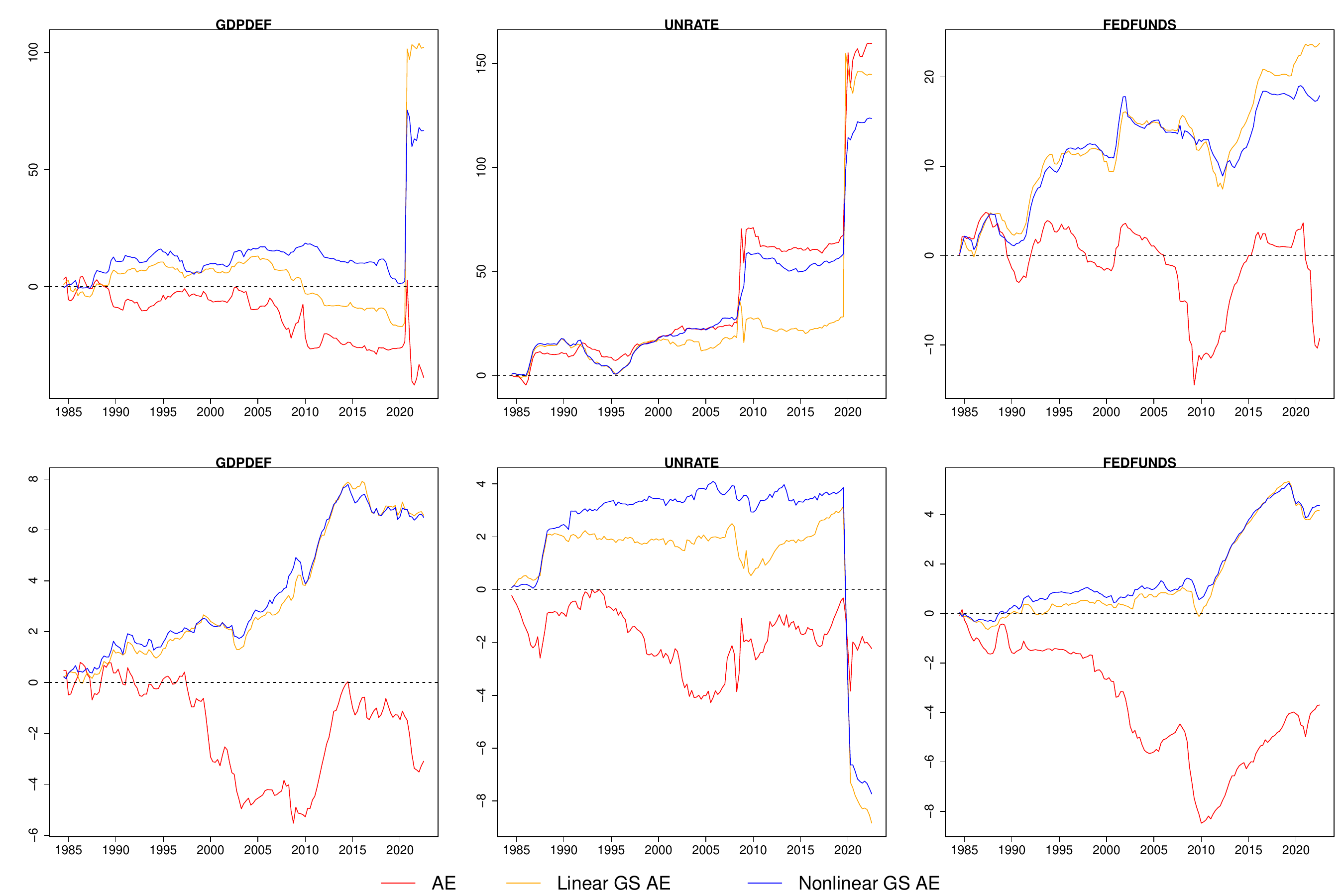}
    \caption{Cumulative ALPL (h=4) of models relative to their PCA counterparts. The top panels consider TIV models and the bottom ones consider TVP models.}
    \label{fig:cumulative_ALPL_h=4}
\end{figure}

Figure \ref{fig:irf_y_full} provides 4 more time points to study the IRFs in the VAR variables. We include those studied in \cite{korobilis2013assessing}, 1975:Q1 and 1996:Q1, and two additional time points corresponding to the chairmanship of Bernanke and Yellen. Overall, we find the transmission of monetary policy shifts gradually. For GDPDEF, the responses peaked earlier as time passed. The UNRATE responded with different degree of uncertainty. The rate at which the FEDFUNDS responses reached to zero decreased over time. 

\begin{figure}[!htbp]
    \centering
    \includegraphics[width=\linewidth]{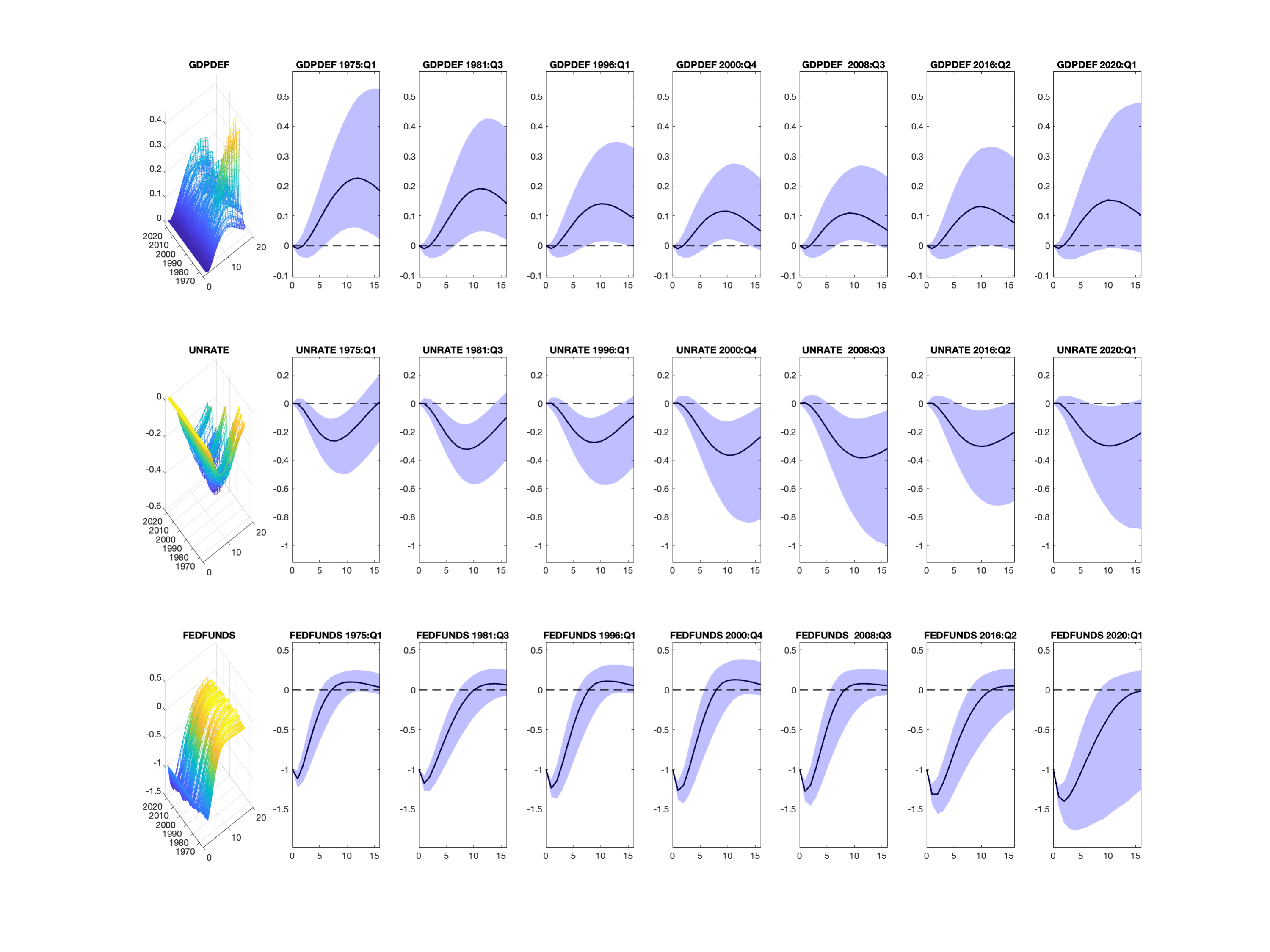}
    \caption{Impulse responses of the VAR variables to a 100 bps decrease in FEDFUNDS. First column shows the medians over time. The rest three columns shows the IRFs with their 68\% credible intervals at 1975:Q1, 1981:Q3, 1996:Q1, 2000:Q4, 2008:Q3, 2016:Q2 and 2020:Q1.}
    \label{fig:irf_y_full}
\end{figure}

\section{Data}
\label{sec:Data}
\begin{center}
\begingroup
\fontsize{8}{9}\selectfont
\begin{longtable}{lllcc}%
\multicolumn{2}{l}{\textbf{Slow Variables}} & & &   \\
\toprule
 & \textbf{Name} & \textbf{Description} & \textbf{Category} & \textbf{Code} \\
\hline
1 & GDPC1 & Real Gross Domestic Product & 1 & 50 \\
2 & PCECC96 & Real Personal Consumption Expenditures & 1 & 50 \\
3 & PCDGx & Real personal consumption expenditures: Durable goods & 1 & 50 \\
4 & PCESVx & Real Personal Consumption Expenditures: Services & 1 & 50 \\
5 & PCNDx & Real Personal Consumption Expenditures: Nondurable Goods & 1 & 50 \\
6 & GPDIC1 & Real Gross Private Domestic Investment & 1 & 50 \\
7 & FPIx & Real private fixed investment & 1 & 50 \\
8 & Y033RC1Q027SBEAx & Real Gross Private Domestic Investment: Fixed Investment: Nonresidential Equip & 1 & 50\\
9 & PNFIx & Real private fixed investment: Nonresidential & 1 & 50\\
10 & PRFIx & Real private fixed investment: Residential & 1 & 50\\
11 & A014RE1Q156NBEA & Shares of gross domestic product: Change in private inventories in private inventories & 1 & 1\\
12 & GCEC1 & Real Government Consumption Expenditures \& Gross Investment & 1 & 50 \\
13 & A823RL1Q225SBEA & Real Government Consumption Expenditures and Gross Investment: Federal & 1 & 1 \\
14 & FGRECPTx & Real Federal Government Current Receipts & 1 & 50 \\
15 & SLCEx & Real government state and local consumption expenditures & 1 & 50 \\
16 & EXPGSC1 & Real Exports of Goods \& Services, 3 Decimal & 1 & 50 \\
17 & IMPGSC1 & Real Imports of Goods \& Services & 1 & 50 \\
18 & DPIC96 & Real Disposable Personal Income & 1 & 50 \\
19 & OUTNFB & Nonfarm Business Sector: Real Output & 1 & 50 \\
20 & OUTBS & Business Sector: Real Output & 1 & 50 \\
21 & INDPRO & Industrial Production Index & 2 & 50 \\
22 & IPFINAL & Industrial Production: Final Products & 2 & 50 \\
23 & IPCONGD &  Industrial Production: Consumer Goods & 2 & 50 \\
24 & IPMAT &  Industrial Production: Materials & 2 & 50 \\
25 & IPDMAT & Industrial Production: Durable Materials & 2 & 50 \\
26 & IPNMAT & Industrial Production: Nondurable Materials & 2 & 50 \\
27 & IPDCONGD & Industrial Production: Durable Consumer Good & 2 & 50 \\
28 & IPB51110SQ & Industrial Production: Durable Goods: Automotive products & 2 & 50 \\
29 & IPNCONGD & Industrial Production: Nondurable Consumer Goods & 2 & 50 \\
30 & IPBUSEQ & Industrial Production: Business Equipment & 2 & 50 \\
31 & IPB51220SQ & Industrial Production: Consumer energy products & 2 & 50 \\
32 & CUMFNS & Capacity Utilization: Manufacturing & 2 & 1 \\
33 & IPMANSICS & Industrial Production: Manufacturing & 2 & 50 \\
34 & IPB51222S &  Industrial Production: Residential Utilities & 2 & 50 \\
35 & IPFUELS & Industrial Production: Fuel & 2 & 50 \\
36 & PAYEMS & All Employees: Total nonfarm  & 3 & 50 \\
37 & USPRIV & All Employees: Total Private Industries & 3 & 50 \\
38 & MANEMP & All Employees: Manufacturing & 3 & 50 \\
39 & SRVPRD & All Employees: Service-Providing Industries & 3 & 50 \\
40 & USGOOD & All Employees: Goods-Producing Industries & 3 & 50 \\
41 & DMANEMP & All Employees: Durable goods & 3 & 50 \\
42 & NDMANEMP & All Employees: Nondurable goods & 3 & 50 \\
43 & USCONS & All Employees: Construction & 3 & 50 \\
44 & USEHS & All Employees: Financial Activities & 3 & 50 \\
45 & USFIRE &  All Employees: Financial Activities & 3 & 50 \\
46 & USINFO & All Employees: Information Services & 3 & 50 \\
47 & USPBS & All Employees: Professional \& Business Services & 3 & 50 \\
48 & USLAH & All Employees: Leisure \& Hospitality & 3 & 50 \\
49 & USSERV & All Employees: Other Services  & 3 & 50 \\
50 & USMINE & All Employees: Mining and logging & 3 & 50 \\
51 & USTPU & All Employees: Trade, Transportation \& Utilities & 3 & 50 \\
52 & USGOVT & All Employees: Government & 3 & 50 \\
53 & USTRADE & All Employees: Retail Trade & 3 & 50 \\
54 & USWTRADE & All Employees: Wholesale Trade  & 3 & 50 \\
55 & CES9091000001 & All Employees: Government: Federal & 3 & 50 \\
56 & CES9092000001 & All Employees: Government: State Government & 3 & 50 \\
57 & CES9093000001 & All Employees: Government: Local Government & 3 & 50 \\
58 & CE16OV & Civilian Employment & 3 & 50 \\
59 & CIVPART &  Civilian Labor Force Participation Rate & 3 & 1 \\
60 & UNRATE & Civilian Unemployment Rate & 3 & 1 \\ 
61 & UNRATESTx & Unemployment Rate less than 27 weeks & 3 & 1 \\
62 & UNRATELTx & Unemployment Rate for more than 27 week & 3 & 1 \\
63 & LNS14000012 & Unemployment Rate - 16 to 19 years & 3 & 1 \\
64 & LNS14000025 & Unemployment Rate - 20 years and over, Men & 3 & 1 \\
65 & LNS14000026 & Unemployment Rate - 20 years and over, Women & 3 & 1 \\
66 & UEMPLT5 & Number of Civilians Unemployed - Less Than 5 Weeks  & 3 & 50 \\
67 & UEMP5TO14 & Number of Civilians Unemployed for 5 to 14 Weeks & 3 & 50 \\
68 & UEMP15T26 &  Number of Civilians Unemployed for 15 to 26 Weeks & 3 & 50 \\
69 & UEMP27OV & Number of Civilians Unemployed for 27 Weeks and Over & 3 & 50 \\
70 & AWHMAN &  Average Weekly Hours of Prod and Nonsuperv Employees: Manufacturing & 3 & 1 \\
71 & AWOTMAN & Avg Weekly Overtime Hours of Prod and Nonsuperv Employees: Manufacturing & 3 & 1 \\
72 & HWIx & Help-Wanted Index & 3 & 1 \\
73 & CES0600000007 & Average Weekly Hours of Production and Nonsupervisory Employees: Goods-Producing & 3 & 1 \\
74 & CLAIMSx & Initial Claims & 3 & 50 \\
75 & HOUST & Housing Starts: Total: New Privately Owned Housing Units Started & 4 & 50 \\
76 & HOUST5F & Privately Owned Housing Starts: 5-Unit Structures or More & 4 & 50 \\
77 & PERMIT &  New Private Housing Units Authorized by Building Permits & 4 & 50 \\
78 & HOUSTMW & Housing Starts in Midwest Census Region  & 4 & 50 \\
79 & HOUSTNE & Housing Starts in Northeast Census Region & 4 & 50 \\
80 & HOUSTS & Housing Starts in South Census Region & 4 & 50 \\
81 & HOUSTW & Housing Starts in West Census Region  & 4 & 50 \\
82 & RSAFSx & Real Retail and Food Services Sales & 5 & 50 \\
83 & AMDMNOx & Real Manufacturers’ New Orders: Durable Goods & 5 & 50 \\
84 & AMDMUOx & Real Value of Manufacturers Unfilled Orders for Durable Goods Industries & 5 & 50 \\
85 & BUSINVx & Total Business Inventories & 5 & 50 \\
86 & ISRATIOx & Total Business: Inventories to Sales Ratio & 5 & 1 \\
87 & GDPDEF & Gross Domestic Product: Implicit Price Deflator & 6 & 1 \\
88 & PCECTPI & Pers Cons Ex: Chain-type Price Index & 6 & 50 \\
89 & PCEPILFE & Pers Cons Exp Excluding Food and Energy & 6 & 50 \\
90 & GDPCTPI & Gross Domestic Product: Chain-type Price Index & 6 & 50 \\
91 & GPDICTPI & Gross Private Domestic Investment: Chain-type Price Index & 6 & 50 \\
92 & IPDBS & Business Sector: Implicit Price Deflator & 6 & 50 \\
93 & DGDSRG3Q086SBEA & Pers Cons Exp: Goods  & 6 & 50 \\
94 & DDURRG3Q086SBEA & Pers Cons Exp: Durable goods & 6 & 50 \\
95 & DSERRG3Q086SBEA & Pers Cons Exp: Services & 6 & 50 \\
96 & DNDGRG3Q086SBEA & Pers Cons Exp: Nondurable goods & 6 & 50 \\
97 & DHCERG3Q086SBEA & Pers Cons Exp: Services: Household consumption
expenditures & 6 & 50 \\
98 & DMOTRG3Q086SBEA & Pers Cons Exp: Durable goods: Motor vehicles and
parts & 6 & 50 \\
99 & DFDHRG3Q086SBEA & Pers Cons Exp: Durable goods: Furnishings and
durable household equipment & 6 & 50 \\
100 & DREQRG3Q086SBEA & Pers Cons Exp: Durable goods: Recreational goods
and vehicles & 6 & 50 \\
101 & DODGRG3Q086SBEA & Pers Cons Exp: Durable goods: Other durable goods & 6 & 50 \\
102 & DFXARG3Q086SBEA & Pers Cons Exp: Food and beverages for off-premises cons & 6 & 50 \\
103 & DCLORG3Q086SBEA & Pers Cons Exp: Nondurable goods: Clothing and footwear & 6 & 50 \\
104 & DGOERG3Q086SBEA & Pers Cons Exp: Nondurable goods: Gasoline and other energy goods & 6 & 50 \\
105 & DONGRG3Q086SBEA & Pers Cons Exp: Nondurable goods: Other nondurable goods & 6 & 50 \\
106 & DHUTRG3Q086SBEA & Pers Cons Exp: Services: Housing and utilities & 6 & 50 \\
107 & DHLCRG3Q086SBEA & Pers Cons Exp: Services: Health care & 6 & 50 \\
108 & DTRSRG3Q086SBEA & Pers Cons Exp: Transportation services & 6 & 50 \\
109 & DRCARG3Q086SBEA & Pers Cons Exp: Recreation services & 6 & 50 \\
110 & DFSARG3Q086SBEA & Pers Cons Exp: Recreation services & 6 & 50 \\
111 & DIFSRG3Q086SBEA & Pers Cons Exp: Services: Food services and accommodations & 6 & 50 \\
112 & DOTSRG3Q086SBEA & Pers Cons Exp: Financial services and insurance & 6 & 50 \\
113 & CPIAUCSL & Consumer Price Index for All Urban Consumers: All Items & 6 & 50 \\
114 & CPILFESL &  Consumer Price Index for All Urban Consumers: All Items Less Food \& Energy & 6 & 50 \\
115 & WPSFD49207 & Producer Price Index by Commodity for Finished Goods & 6 & 50 \\
116 & PPIACO & Producer Price Index for All Commodities & 6 & 50 \\
117 & WPSFD49502 & Producer Price Index by Commodity for Finished Consumer Goods & 6 & 50 \\
118 & WPSFD4111 & Producer Price Index by Commodity for Finished Consumer Foods & 6 & 50 \\
119 & PPIIDC & Producer Price Index by Commodity Industrial Commodities & 6 & 50 \\
120 & WPSID61 & PPI by Commodity Intermediate Materials: Supplies \& Components & 6 & 50 \\
121 & WPU0561 & Producer Price Index by Commodity for Fuels and Related Products and Power & 6 & 50 \\
122 & OILPRICEx & Real Crude Oil Prices: West Texas Intermediate (WTI) - Cushing, Oklahoma & 6 & 50 \\
123 & WPSID62 & Producer Price Index: Crude Materials for Further Processing & 6 & 50 \\
124 & PPICMM & PPI: Commodities: Metals and metal products: Primary nonferrous metals & 6 & 50 \\
125 & CPIAPPSL & Consumer Price Index for All Urban Consumers: Apparel & 6 & 50 \\
126 & CPITRNSL &  Consumer Price Index for All Urban Consumers: Transportation & 6 & 50 \\
127 & CPIMEDSL &  Consumer Price Index for All Urban Consumers: Medical Care & 6 & 50 \\
128 & CUSR0000SAC & Consumer Price Index for All Urban Consumers: Commodities & 6 & 50 \\
129 & CES2000000008x & Real Average Hourly Earnings of Prod and Nonsuperv Employees: Construction & 7 & 50 \\
130 & CES3000000008x & Real Average Hourly Earnings of Prod and Nonsuperv Employees: Manufacturing & 7 & 50\\
131 &  COMPRNFB & Nonfarm Business Sector: Real Compensation Per Hour (Index 2012=100) & 7 & 50 \\
132 & CES0600000008 & Average Hourly Earnings of Production and Nonsupervisory Employees & 7 & 50\\
\bottomrule
\caption{Description of slow variables. Transformation code: 1 - level; 5 - first differences of logarithms; 7 - $\Delta(x_t/x_{t-1}-1)$; 50 - year-over-year log difference.}
\label{data appendix}
\end{longtable}
\endgroup

\end{center}
\begin{center}
\begingroup
\fontsize{8}{9}\selectfont
\begin{longtable}{lllcc}%
\multicolumn{2}{l}{\textbf{Fast Variables}} & & &  \\
\toprule
 & \textbf{Name} & \textbf{Description} & \textbf{Category} & \textbf{Code} \\
\hline
1 & FEDFUNDS &  Effective Federal Funds Rate & 8 & 1 \\
2 & TB3MS & 3-Month Treasury Bill: Secondary Market Rate & 8 & 1 \\
3 & TB6MS & 6-Month Treasury Bill: Secondary Market Rate & 8 & 1 \\
4 & GS1 & 1-Year Treasury Constant Maturity Rate & 8 & 1 \\
5 & GS10 & 10-Year Treasury Constant Maturity Rate & 8 & 1 \\
6 & AAA &  Moodys Seasoned Aaa Corporate Bond Yield & 8 & 1 \\
7 & BAA & Moodys Seasoned Baa Corporate Bond Yield & 8 & 1 \\
8 & BAA10YM & Moodys Seasoned Baa Corporate Bond Yield Rel. to Yield on 10-Year Treasury & 8 & 1 \\
9 & TB6M3Mx & 6-Month Treasury Bill Minus 3-Month Treasury Bill, secondary market & 8 & 1 \\
10 & GS1TB3Mx &  1-Year Treasury Constant Maturity Minus 3-Month Treasury Bill, second market & 8 & 1 \\
11 & GS10TB3Mx & 10-Year Treasury Constant Maturity Minus 3-Month Treasury Bill, second market & 8 & 1 \\
12 & CPF3MTB3Mx & 3-Month Commercial Paper Minus 3-Month Treasury Bill, second market & 8 & 1 \\
13 & GS5 & 5-Year Treasury Constant Maturity Rate & 8 & 1 \\
14 & TB3SMFFM &  3-Month Treasury Constant Maturity Minus Federal Funds Rate & 8 & 1 \\
15 & T5YFFM &  5-Year Treasury Constant Maturity Minus Federal Funds Rate & 8 & 1 \\
16 & AAAFFM & Moodys Seasoned Aaa Corporate Bond Minus Federal Funds Rate & 8 & 1 \\
17 & M1REAL & Real M1 Money Stock & 9 & 50 \\
18 & M2REAL & Real M2 Money Stock & 9 & 50 \\
19 & BUSLOANSx & Real Commercial and Industrial Loans, All Commercial Banks & 9 & 50 \\
20 & CONSUMERx & Real Consumer Loans at All Commercial Banks & 9 & 50 \\
21 & NONREVSLx &  Total Real Nonrevolving Credit Owned and Securitized, Outstanding & 9 & 50 \\
22 & REALLNx & Real Real Estate Loans, All Commercial Banks & 9 & 50 \\
23 & TOTALSLx & Total Consumer Credit Outstanding & 9 & 50 \\
24 & TOTRESNS & Total Reserves of Depository Institutions & 9 & 50 \\
25 & NONBORRES & Reserves Of Depository Institutions, Nonborrowed & 9 & 7 \\
26 & DTCOLNVHFNM & Consumer Motor Vehicle Loans Outstanding Owned by Finance Companies & 9 & 50 \\
27 & DTCTHFNM & Total Consumer Loans and Leases Outstanding Owned and Sec by Finance Comp & 9 & 50 \\
28 & INVEST & Securities in Bank Credit at All Commercial Banks & 9 & 50 \\
29 & TABSHNOx & Real Total Assets of Households and Nonprofit Organizations & 10 & 50 \\
30 & EXSZUSx & Switzerland / U.S. Foreign Exchange Rate & 11 & 50 \\
31 & EXJPUSx & Japan / U.S. Foreign Exchange Rate & 11 & 50 \\
32 & EXUSUKx & U.S. / U.K. Foreign Exchange Rate & 11 & 50 \\
33 & EXCAUSx &  Canada / U.S. Foreign Exchange Rate & 11 & 50 \\
34 & S\&P 500  &  S\&Ps Common Stock Price Index: Composite & 12 & 5 \\
35 & S\&P: indust & S\&Ps Common Stock Price Index: Industrials & 12 & 50 \\
36 & S\&P div yield & S\&Ps Composite Common Stock: Dividend Yield & 12 & 1 \\
\bottomrule
\caption{Description of fast variables. Transformation code: 1 - level; 5 - first differences of logarithms; 7 - $\Delta(x_t/x_{t-1}-1)$; 50 - year-over-year log difference.}
\label{data appendix}
\end{longtable}
\endgroup

\end{center}


\end{document}